\documentclass{article}

\usepackage[accepted]{icml2021}

\usepackage[utf8]{inputenc} 
\usepackage[T1]{fontenc}
\usepackage{hyperref}   
\usepackage{url}
\usepackage{booktabs}   
\usepackage{amsfonts}   
\usepackage{nicefrac}   
\usepackage{microtype}  

\usepackage{xcolor}   
\usepackage{acronym}  
\usepackage{amsthm}   
\usepackage{multirow} 
\usepackage{pifont} 
\usepackage{threeparttable}   

\usepackage{amsmath, amssymb,amsthm}
\usepackage{mathtools}

\newcommand{\nats}{\mathbb{N}}
\newcommand{\reals}{\mathbb{R}}

\DeclarePairedDelimiter{\norm}{\lVert}{\rVert}
\DeclareMathOperator{\sigmoid}{\sigma}
\DeclareMathOperator{\diag}{diag}

\DeclareMathOperator{\softmax}{\mathrm{softmax}}
\DeclareMathOperator{\relu}{\mathrm{ReLU}}
\DeclareMathOperator{\bigO}{\mathcal{O}}

\renewcommand{\vec}[1]{\boldsymbol{\MakeLowercase{#1}}}
\newcommand{\mat}[1]{\boldsymbol{\MakeUppercase{#1}}}
\DeclareMathAlphabet{\mathsfit}{\encodingdefault}{\sfdefault}{m}{sl}
\SetMathAlphabet{\mathsfit}{bold}{\encodingdefault}{\sfdefault}{bx}{n}
\newcommand{\ten}[1]{\boldsymbol{\mathsfit{\MakeUppercase{#1}}}}

\let\normalised\tilde
\newcommand{\mass}[1]{{#1}}
\newcommand{\aux}[1]{#1}

\newtheorem{corollary}{Corollary}%

\newcommand{\timestep}[1]{^{#1}}

\newtheorem{theorem}{Theorem}

\acrodef{mae}[MAE]{Mean Absolute Error}
\acrodef{mse}[MSE]{Mean Squared Error}
\acrodef{rmse}[RMSE]{Root \acs{mse}}
\acrodef{nlp}[NLP]{Natural Language Processing}
\acrodef{hmm}[HMM]{Hidden Markov Model}
\acrodef{rnn}[RNN]{Recurrent Neural Network}
\acrodef{lstm}[LSTM]{Long Short-Term Memory}
\acrodef{mclstm}[MC-LSTM]{Mass-Conserving \acs{lstm}}
\acrodef{mcfc}[MC-FC]{Mass-Conserving Fully Connected}
\acrodef{nalu}[NALU]{Neural ALU}
\acrodef{nau}[NAU]{Neural Addition Unit}
\acrodef{nac}[NAC]{Neural Accumulator}
\acrodef{nmu}[NMU]{Neural Multiplication Unit}
\acrodef{cnn}[CNN]{Convolutional Neural Network}
\acrodef{lrp}[LRP]{Layer-wise Relevance Propagation}
\acrodef{gru}[GRU]{Gated Recurrent Units}
\acrodef{mac}[MAC]{Multiply-Accumulate operation}
\acrodef{hnn}[HNN]{Hamiltonian Neural Network}
\acrodef{gdl}[GDL]{Geometric Deep Learning}

\icmltitlerunning{\acs{mclstm}}

\begin{document}

\twocolumn[
\icmltitle{\acs{mclstm}: \acl{mclstm}}



\icmlsetsymbol{equal}{*}

\begin{icmlauthorlist}
\icmlauthor{Pieter-Jan Hoedt}{equal,jku}
\icmlauthor{Frederik Kratzert}{equal,jku}
\icmlauthor{Daniel Klotz}{jku}
\icmlauthor{Christina Halmich}{jku}
\icmlauthor{Markus Holzleitner}{jku}
\icmlauthor{Grey Nearing}{goo}
\icmlauthor{Sepp Hochreiter}{jku,iarai}
\icmlauthor{Günter Klambauer}{jku}
\end{icmlauthorlist}

\icmlaffiliation{jku}{ELLIS Unit Linz, LIT AI Lab, Institute for Machine Learning, Johannes Kepler University Linz, Austria}
\icmlaffiliation{goo}{Google Research, Mountain View, CA, USA}
\icmlaffiliation{iarai}{Institute of Advanced Research in Artificial Intelligence ({IARAI})}

\icmlcorrespondingauthor{Pieter-Jan Hoedt}{hoedt@ml.jku.at}
\icmlcorrespondingauthor{Frederik Kratzert}{kratzert@ml.jku.at}

\icmlkeywords{Deep Learning, LSTM, RNN, inductive bias, mass-conservation, neural arithmetic units, hydrology}

\vskip 0.3in
]



\printAffiliationsAndNotice{\icmlEqualContribution} 


\begin{abstract}
    The success of \acp{cnn} in computer vision is mainly driven by their strong inductive bias,
    which is strong enough to allow \acp{cnn} to solve vision-related tasks with random weights, meaning without learning.
    Similarly, \ac{lstm} has a strong inductive bias toward storing information over time.
    However, many real-world systems are governed by conservation laws, which lead to the redistribution of particular quantities --- e.g.\ in physical and economical systems.
    Our novel \ac{mclstm} adheres to these conservation laws by extending the inductive bias of \ac{lstm} to model the redistribution of those stored quantities.
    \acp{mclstm} set a new state-of-the-art for neural arithmetic units at learning arithmetic operations, such as addition tasks, which have a strong conservation law, as the sum is constant over time.
    Further, \ac{mclstm} is applied to traffic forecasting, modeling a damped pendulum, and a large benchmark dataset in hydrology, where it sets a new state-of-the-art for predicting peak flows. 
    In the hydrology example, we show that \ac{mclstm} states correlate with real world processes and are therefore interpretable.
\end{abstract}

\acresetall

\section{Introduction}
\label{sec:intro}
\paragraph{Inductive biases enabled the success of CNNs and LSTMs.}
One of the greatest success stories of deep learning are \acp{cnn}
\citep{Fukushima:80,LeCun:98,Schmidhuber2015deep, Lecun2015deep},
whose proficiency can be attributed to their strong inductive bias 
toward visual tasks \citep{Cohen:17,Gaier:19nips}.
The effect of this inductive bias has been demonstrated by
\acp{cnn} that solve vision-related tasks with random weights, meaning without learning \citep{He:16,Gaier:19nips,Ulyanov:20}.
Another success story is 
\ac{lstm} \citep{hochreiter91, hochreiterS97}, which has a strong
inductive bias toward storing information through its memory cells.
This inductive bias allows \ac{lstm} to excel at speech, text, and
language tasks \citep{sutskever2014sequence,bohnet2018morphosyntactic,kochkina2017turing,liu2019bidirectional},
as well as timeseries prediction. 
Even with random weights and only a learned linear output 
layer, \ac{lstm} is better at predicting timeseries 
than reservoir methods \citep{Schmidhuber:07}.
In a seminal paper on biases in machine learning, \citet{Mitchell:80} stated that
\textit{``biases and initial knowledge are at the heart of the ability to generalize
beyond observed data''}.
Therefore, choosing an appropriate architecture and inductive bias 
for neural networks is key to generalization.

\paragraph{Mechanisms beyond storing are required for real-world applications.}
While \ac{lstm} can store information over time, real-world applications 
require mechanisms that go beyond storing.
Many real-world systems are governed by conservation laws related to 
mass, energy, momentum, charge, or particle counts, which
are often expressed through continuity equations.
In physical systems, different types of energies, mass or particles have to be conserved
\citep{evans2005nonequilibrium,rabitz1999efficient, vanderschaft1996mathematical}, 
in hydrology it is the amount of water
\citep{freeze1969blueprint, beven2011rainfall}, 
in traffic and transportation the number of vehicles
\citep{vanajakshi2004loop,xiao2020new,zhao2017lstm}, 
and in logistics the amount of goods, money or 
products.
A real-world task could be to predict outgoing goods from a warehouse based on a general state 
of the warehouse, i.e., how many goods are in storage, and incoming supplies.
If the predictions are not precise, then they do not lead
to an optimal control of the production process.
For modeling such systems, certain inputs must be conserved
but also redistributed across storage locations within the system.
We will refer to conserved inputs as \emph{mass}, but note that 
this can be any type of conserved quantity.
We argue that for modeling such systems, 
specialized mechanisms should be used to represent 
locations \& whereabouts, objects, or storage \& 
placing locations and thus enable conservation.

\paragraph{Conservation laws should pervade machine learning models in 
the physical world.} 
Since a large part of machine learning models are developed 
to be deployed in the real 
world, in which conservation laws are omnipresent rather than the exception,
these models should adhere to them automatically and benefit from them. 
However, standard deep learning approaches struggle at conserving
quantities across layers or timesteps 
\citep{beucler2019achieving,greydanusDY19,song1996solving,yitian2003modeling},
and often solve a task by exploiting spurious correlations \citep{szegedy2013intriguing,lapuschkin2019unmasking}.
Thus, an inductive bias of deep learning approaches 
via mass conservation over time
in an open system, where mass can be added and removed, could lead to 
a higher generalization performance than standard deep learning for
the above-mentioned tasks.




\paragraph{A mass-conserving \ac{lstm}.} 
In this work, we introduce \ac{mclstm}, a variant of \ac{lstm} that enforces mass conservation by design. 
\ac{mclstm} is a recurrent neural network with an architecture 
inspired by the gating mechanism in \acp{lstm}. \ac{mclstm} has a
strong inductive bias to guarantee the conservation of mass.
This conservation is implemented by means of left-stochastic matrices, which ensure the sum 
of the memory cells in the network represents the current mass in the system. 
These left-stochastic matrices also enforce the mass to be conserved through time.
The \ac{mclstm} gates operate as control units on mass flux.
Inputs are divided into a subset of \emph{mass inputs}, 
which are propagated through time and are conserved,
and a subset of \emph{auxiliary inputs}, 
which serve as inputs to the gates for controlling mass fluxes.
We demonstrate that \acp{mclstm} excel at tasks where 
conservation of mass is required and that it is highly apt at solving
real-world problems in the physical domain. 

\paragraph{Contributions.} 
We propose a novel neural network architecture based on \ac{lstm} that conserves quantities, such as mass, energy, or count, of a specified set of inputs. 
We show properties of this novel architecture, called \ac{mclstm}, and demonstrate that these properties render it a powerful neural arithmetic unit.
Further, we show its applicability in real-world areas of traffic forecasting and modeling the damped pendulum. 
In hydrology, large-scale benchmark experiments reveal that \ac{mclstm} has powerful predictive quality and can supply interpretable representations. 

\section{\acl{mclstm}}
\label{sec:model}

The original \ac{lstm}
introduced memory cells to \acp{rnn}, which alleviate 
the vanishing gradient problem \citep{hochreiter91}. 
This is achieved by means of a fixed recurrent self-connection of the memory cells. 
If we denote the values in the memory cells at time $t$ by $\vec{c}
   \timestep{t}$, this recurrence can be formulated as 
\begin{equation}
    \label{eq:cell_accumulation}
    \vec{c}\timestep{t} = \vec{c}\timestep{t-1} + f(\vec{x}\timestep{t}, \vec{h}\timestep{t-1}),
\end{equation}
where $\vec{x}$ and $\vec{h}$ are, respectively, the forward inputs and 
recurrent inputs, and $f$ is some function that computes the increment for the memory cells.
Here, we used the original formulation of \ac{lstm} without forget gate \citep{hochreiterS97}, but in 
all experiments we also consider \ac{lstm} with forget gate \citep{gersSC00}.

\acp{mclstm} modify this recurrence to guarantee the conservation 
of the mass input.
The key idea is to use the memory cells from \acp{lstm} as mass accumulators, or mass storage. 
The conservation law is implemented by three architectural changes. 
First, 
the increment, computed by $f$ in Eq.~\eqref{eq:cell_accumulation}, has to 
distribute mass from inputs into accumulators. 
Second, the mass that leaves \ac{mclstm} must also disappear from the 
accumulators. 
Third, mass has to be redistributed between mass accumulators.
These changes mean that all gates explicitly represent mass fluxes.

 \begin{figure}
    \centering
    \includegraphics[width=\linewidth]{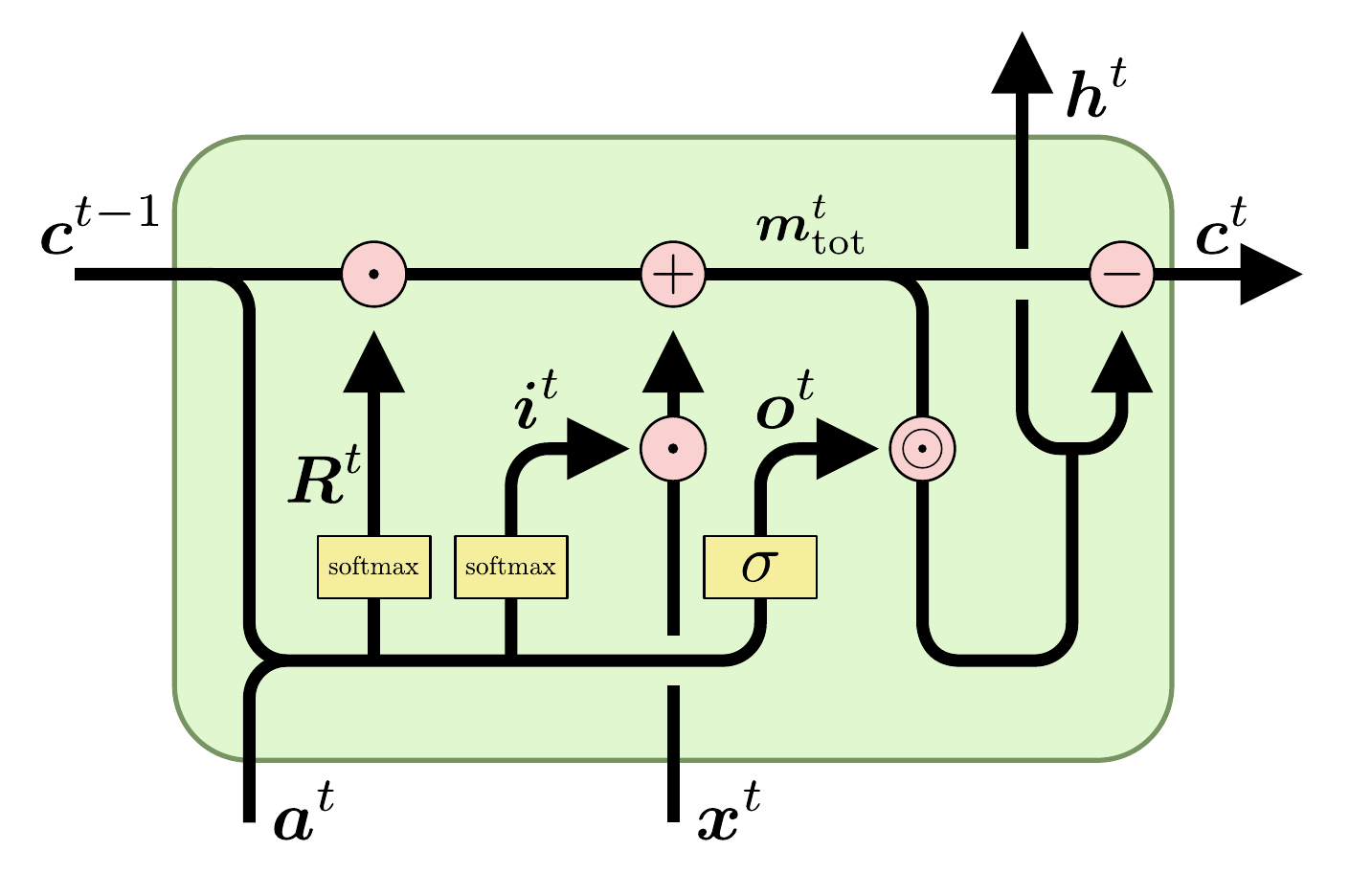}
    \caption{Schematic representation of the main operations in the 
    \ac{mclstm} architecture \citep[adapted from:][]{colah15}.}
\end{figure}
 
Since, in general, not all inputs must be conserved, we distinguish 
between \emph{mass} inputs, $\mass{\vec{x}}$, and \emph{auxiliary} 
inputs, $\aux{\vec{a}}$. The former represents the quantity to be conserved
and will fill the mass accumulators in \ac{mclstm}. The auxiliary inputs
are used to control the gates. To keep the notation uncluttered, and without loss of generality, 
we use a single mass input at each timestep, $\mass{x}\timestep{t}$, to introduce the architecture.

The forward pass of \ac{mclstm} at timestep $t$ can be specified as follows:
\begin{align}
    \label{eq:mclstm:mass}
    \vec{m}_\mathrm{tot}\timestep{t}  &= \mat{R}\timestep{t} \cdot \vec{c}\timestep{t-1} + \vec{i}\timestep{t} \cdot \mass{x}\timestep{t} \\
    \label{eq:mclstm:cell_update}
    \vec{c}\timestep{t} &= (\vec{1} - \vec{o}\timestep{t}) \odot \vec{m}_\mathrm{tot}\timestep{t}  \\
    \label{eq:mclstm:out_update}
    \vec{h}\timestep{t} &= \vec{o}\timestep{t} \odot \vec{m}_\mathrm{tot}\timestep{t},
\end{align}

where $\vec{i}\timestep{t}$ and $\vec{o}\timestep{t}$ are the input- and 
output gates, respectively,  and $\mat{R}$ is a 
positive left-stochastic matrix, i.e., $\vec{1}^T \cdot \mat{R} = \vec{1}^T$, 
for redistributing mass in the accumulators. 
The \emph{total mass} $\vec{m}_\mathrm{tot}$ is 
the \emph{redistributed mass}, $\mat{R}\timestep{t} \cdot \vec{c}\timestep{t-1}$, plus 
the \emph{mass influx}, or new mass, $\vec{i}\timestep{t} \cdot \mass{x}\timestep{t}$.
The current mass in the system is stored in $\vec{c}\timestep{t}$.
Finally, $\vec{h}\timestep{t}$ is the mass leaving the system.

Note the differences between Eq.~\eqref{eq:cell_accumulation} 
and Eq.~\eqref{eq:mclstm:cell_update}. First, the increment of the memory cells no
longer depends on $\vec{h}\timestep{t}$. Instead, mass inputs are 
distributed by means of the normalized 
$\vec{i}$ (see Eq.~\ref{eq:mclstm:in_gate}). Furthermore, $\mat{R}\timestep{t}$ 
replaces the implicit identity matrix of \ac{lstm} to redistribute mass among memory cells.
Finally, Eq.~\eqref{eq:mclstm:cell_update} 
introduces $\vec{1} - \vec{o}\timestep{t}$ as a forget gate on the total mass, $\vec{m}_\mathrm{tot}$. 
Together with Eq.~\eqref{eq:mclstm:out_update}, this assures that no outgoing mass is stored in the accumulators.
This formulation has some similarity to \ac{gru} \citep{cho2014learning}, 
however \ac{mclstm} gates are used to split off the output instead of mixing 
the old and new cell state.

\paragraph{Basic gating and redistribution.} 
The \ac{mclstm} gates at timestep $t$ are computed as follows:
\begin{align}
    \label{eq:mclstm:in_gate}
    \vec{i}\timestep{t} &= \softmax (\mat{W}_\mathrm{i} \cdot \aux{\vec{a}}\timestep{t} + \mat{U}_\mathrm{i} \cdot \frac{\vec{c}\timestep{t-1}}{\norm{\vec{c}\timestep{t-1}}_1} + \vec{b}_\mathrm{i}) \\
    \label{eq:mclstm:out_gate}
    \vec{o}\timestep{t} &= \sigmoid(\mat{W}_\mathrm{o} \cdot \aux{\vec{a}}\timestep{t} + \mat{U}_\mathrm{o} \cdot \frac{\vec{c}\timestep{t-1}}{\norm{\vec{c}\timestep{t-1}}_1} + \vec{b}_\mathrm{o}) \\
    \label{eq:mclstm:redistribution}
    \mat{R}\timestep{t} &= \softmax(\mat{B}_\mathrm{r}  ),
\end{align}
where 
the $\softmax$ operator is applied column-wise, 
$\sigma$ is the logistic sigmoid function, and
$\mat{W}_\mathrm{i}$, $\vec{b}_\mathrm{i}$,
$\mat{W}_\mathrm{o}$, $\vec{b}_\mathrm{o}$, and $\mat{B}_\mathrm{r}$
are learnable model parameters. 
The normalization of the input gate and redistribution is required to obtain mass conservation.
Note that this can also be achieved by other means than using
the softmax function. For example, an alternative way 
to ensure a column-normalized matrix $\mat{R}\timestep{t}$ 
is to use a normalized logistic, 
$\normalised{\sigmoid}(r_{kj}) = \frac{\sigmoid(r_{kj})}{\sum_n \sigmoid(r_{kn})}$.
Also note that \acp{mclstm} directly compute the gates from the memory cells.
This is in contrast with the original \ac{lstm}, which uses the activations from the previous time step.
In this sense, \ac{mclstm} relies on peephole connections \citep{gersS00}, instead of the activations from the previous timestep for computing the gates. 
The accumulated values from the memory cells, $\vec{c}\timestep{t}$, are normalized to counter saturation of the sigmoids and to supply probability vectors 
that represent the current distribution of the mass across cell states.
We use this variation e.g.\ in our experiments 
with \emph{neural arithmetics} (see Sec.~\ref{sec:experiments:arithmetic}). 

\paragraph{Time-dependent redistribution.} 
It can also be useful to predict a redistribution matrix 
for each sample and timestep, similar to how the gates are computed:
\begin{align}
    \label{eq:R}
    \mat{R}\timestep{t} = \softmax \left( \ten{W}_\mathrm{r} \cdot \aux{\vec{a}}\timestep{t} + \ten{U}_\mathrm{r} \cdot \frac{\vec{c}\timestep{t-1}}{\norm{\vec{c}\timestep{t-1}}_1}+ \mat{B}_\mathrm{r} \right),
\end{align}
where the parameters 
$\ten{W}_\mathrm{r}$ and $\ten{U}_\mathrm{r}$ are weight tensors 
and their multiplications result in $K \times K$ matrices.
Again, the $\softmax$ function is applied column-wise.
This version collapses to a time-independent redistribution matrix
if $\ten{W}_\mathrm{r}$ and $\ten{U}_\mathrm{r}$ are equal to $\mat{0}$. Thus, there exists the option to initialize $\ten{W}_\mathrm{r}$ and $\ten{U}_\mathrm{r}$ with weights that are small in absolute value compared to the weights of $\mat{B}_\mathrm{r}$, to favour learning time-independent redistribution matrices. We use this variant in the hydrology experiments (see Sec.~\ref{sec:hydrology}).

\textbf{Redistribution via a hypernetwork.} 
Even more general, a hypernetwork \citep{schmidhuber1992learning, haDL17} that
we denote with $g$ can be used to procure $\mat{R}$.
The hypernetwork has to produce a column-normalized, square matrix $\mat{R}\timestep{t} = g(\aux{\vec{a}}\timestep{0},\ldots,\aux{\vec{a}}\timestep{t}, \vec{c}\timestep{0}, \ldots, \vec{c}\timestep{t-1})$.
Notably, a hypernetwork can be used to design an \emph{autoregressive} version of \acp{mclstm}, if the network additionally predicts auxiliary inputs for the next time step.
We use this variant in the pendulum experiments (see Sec.~\ref{sec:pedulum}).

%

\section{Properties}
\label{sec:theory}

\paragraph{Conservation.}
\ac{mclstm} guarantees that mass is conserved over time. 
This is a direct consequence 
of connecting memory cells with stochastic matrices. 
The mass conservation ensures 
that no mass can be removed or added implicitly, which makes 
it easier to learn functions that generalize well. The exact meaning 
of mass conservation is formalized in the following Theorem. 

\begin{theorem}[Conservation property]
    \label{thm:conservation}
    Let $m_c\timestep{\tau} = \sum_{k=1}^K c_k\timestep{\tau}$ be the
    mass contained in the system
    and $m_h\timestep{\tau} = \sum_{k=1}^K h_k\timestep{\tau}$ be 
    the mass efflux, or, respectively, the \emph{accumulated mass} in the \ac{mclstm} 
    storage and the outputs at time $\tau$. At any timestep $\tau$, we have:
    \begin{equation}
        \label{eq:massconv}
        m_c\timestep{\tau} = m_c\timestep{0} + \sum_{t = 1}^\tau \mass{x}\timestep{t}- \sum_{t = 1}^\tau m_h\timestep{t}.
    \end{equation}
    That is, the change of mass in the memory cells is the 
    difference between the input and output mass, accumulated over time.
\end{theorem}

The proof is by induction over $\tau$ (see Appendix~\ref{app:proof}). 
Note that it is still possible for input mass to be stored indefinitely in a 
memory cell so that it does not appear at the output. This can be a 
useful feature if not all of the input mass is needed at the output. In this 
case, the network can learn that one cell should operate as a collector for
excess mass in the system. 

\paragraph{Boundedness of cell states.} 
In each timestep $\tau$, the memory cells, $c_k\timestep{\tau}$,
are bounded by the sum of mass inputs $\sum_{t=1}^\tau \mass{x}\timestep{t} + m_c\timestep{0}$, that 
is $|c_k\timestep{\tau}| \leq \sum_{t=1}^\tau \mass{x}\timestep{t} + m_c\timestep{0}$. Furthermore, 
if the series of mass inputs converges, $\lim_{\tau \rightarrow \infty} \sum_{t=1}^\tau \mass{x}\timestep{\tau} = m_x^\infty$, then also the sum of cell states converges (see Appendix, Corollary~\ref{cl:bound}).

\paragraph{Initialization and gradient flow.}  
\ac{mclstm} with $\mat{R}\timestep{t}=\mat{I}$ has a similar
gradient flow to \ac{lstm} with forget gate \citep{gersSC00}.
Thus, the main difference in the gradient flow is determined
by the redistribution matrix $\mat{R}$. The forward pass of
\ac{mclstm} without gates
$\vec{c}\timestep{t}=\mat{R}\timestep{t} \vec{c}\timestep{t-1}$
leads to the following backward expression
$\frac{\partial \vec{c}\timestep{t}}{\partial \vec{c}\timestep{t-1}} = \mat{R}\timestep{t}$.
Hence, \ac{mclstm} should be initialized
with a redistribution matrix close to the identity matrix to ensure
a stable gradient flow as in \acp{lstm}. 
For random redistribution matrices, 
the \emph{circular law theorem for random Markov matrices} \citep{bordenave10circular} can be used to 
analyze the gradient flow in more detail, see Appendix, Section~\ref{sec:circular}.

%

\paragraph{Computational complexity.}
Whereas the gates in a traditional \ac{lstm} are vectors, the input gate and redistribution matrix of an \ac{mclstm} are matrices in the most general case. This means that \ac{mclstm} is, in general, computationally more demanding than \ac{lstm}. Concretely, the forward pass for a single timestep in \ac{mclstm} requires $\bigO(K^3 + K^2 (M + L) + K M L)$~\acp{mac}, whereas \ac{lstm} takes $\bigO(K^2 + K (M + L))$~\acp{mac} per timestep. Here, $M$, $L$ and $K$ are the number of mass inputs, auxiliary inputs and outputs, respectively. When using a time-independent redistribution matrix cf. Eq.~\eqref{eq:mclstm:redistribution}, the complexity reduces to $\bigO(K^2 M + K M L)$~\acp{mac}. An empirical runtime comparison is provided in appendix~\ref{app:runtime}.

\paragraph{Potential interpretability through inductive bias and accessible mass in cell states.}
The representations within the model can be interpreted 
directly as accumulated mass. If one mass 
or energy quantity is known, the \ac{mclstm}  architecture would allow
to force a particular cell state to represent this quantity, which
could facilitate learning and interpretability.
An illustrative example is the case of rainfall runoff modelling, 
where observations, say of the soil moisture or groundwater-state, could be used to guide the learning of an explicit memory cell of \ac{mclstm}.

\section{Special Cases and Related Work}
\label{sec:related}
\paragraph{Relation to Markov chains.} 
In a special case \ac{mclstm} collapses to a \emph{finite Markov chain}, 
when $\vec{c}\timestep{0}$ is a probability vector, the mass input is zero $\mass{x}\timestep{t}=0$ for all $t$,
there is no input and output gate, and the redistribution matrix is constant 
over time $\mat{R}\timestep{t}=\mat{R}$. For finite Markov chains, the dynamics are known to 
converge, if $\mat{R}$ is irreducible (see e.g. \citet[Theorem 3.13.]{Hairer:18}). 
\citet{awiszusR18} aim to model a Markov Chain by having a feed-forward network predict 
the next state distribution given the current state distribution. In order to insert 
randomness to the network, a random seed is appended to the input, which allows to 
simulate Markov processes. Although \acp{mclstm} are closely related to Markov chains, 
they do not explicitly learn the transition matrix, as is the case for Markov chain 
neural networks. \acp{mclstm} would have to learn the transition matrix implicitly.

    \begin{table*}
        \caption{Performance of different models on the \ac{lstm} addition task in terms of the \ac{mse}. \ac{mclstm} significantly (all $p$-values below $.05$) outperforms its competitors, \ac{lstm} (with high initial forget gate bias), \ac{nalu} and \ac{nau}. Error bars represent 95\%-confidence intervals across 100 runs.}
        \label{tab:lstm_addition}
        
        \begin{center}\begin{threeparttable}
            \begin{tabular}{lr@{$\ \pm\ $}lr@{$\ \pm\ $}lr@{$\ \pm\ $}lr@{$\ \pm\ $}lr@{$\ \pm\ $}lr}
                \toprule
                & \multicolumn{2}{c}{reference\tnote{a}} & \multicolumn{2}{c}{seq length\tnote{b}} & \multicolumn{2}{c}{input range\tnote{c}} & \multicolumn{2}{c}{count\tnote{d}} & \multicolumn{2}{c}{combo\tnote{e}} & \texttt{NaN}\tnote{f} \\
                \midrule
                \acs{mclstm} & \textbf{0.004} & 0.003 & \textbf{0.009} & 0.004 & \textbf{0.8} & 0.5 & \textbf{0.6} & 0.4 & \textbf{4.0} & 2.5 & 0 \\
                \acs{lstm} & 0.008 & 0.003 & 0.727 & 0.169 & 21.4 & 0.6 & 9.5 & 0.6 & 54.6 & 1.0 & 0 \\
                \acs{nalu} & 0.060 & 0.008 & 0.059 & 0.009 & 25.3 & 0.2 & 7.4 & 0.1 & 63.7 & 0.6 & 93 \\
                \acs{nau} & 0.248 & 0.019 & 0.252 & 0.020 & 28.3 & 0.5 & 9.1 & 0.2 & 68.5 & 0.8 & 24 \\
                \bottomrule
            \end{tabular}
            \begin{tablenotes}
                \item [a] training regime: \hfill summing 2 out of 100 numbers between 0 and 0.5.
                \item [b] longer sequence lengths: \hfill summing 2 out of 1\,000 numbers between 0 and 0.5.
                \item [c] more \emph{mass} in the input: \hfill summing 2 out of 100 numbers between 0 and 5.0.
                \item [d] higher number of summands: \hfill summing 20 out of 100 numbers between 0 and 0.5.
                \item [e] combination of previous scenarios: \hfill summing 10 out of 500 numbers between 0 and 2.5.
                \item [f] Number of runs that did not converge.
            \end{tablenotes}
        \end{threeparttable}\end{center}
    \end{table*}

\paragraph{Relation to normalizing flows and volume-conserving neural networks.} 
In contrast to \emph{normalizing flows} \citep{rezende2015variational,papamakarios2019normalizing}, 
which transform inputs in each layer
and trace their density through layers or timesteps, \acp{mclstm} transform 
distributions and do not aim to trace individual inputs through timesteps.
Normalizing flows thereby conserve information about the input in the first layer 
and can use the inverted mapping to trace an input back to the initial space. 
\acp{mclstm} are concerned with modeling the changes of the initial distribution 
over time and can guarantee that a multinomial distribution is mapped to a multinomial distribution. 
For \acp{mclstm} without gates, 
the sequence of cell states $\vec{c}\timestep{0}, \ldots, \vec{c}\timestep{T}$ 
constitutes a \emph{normalizing flow} if an initial distribution $p_{0}(\vec{c}\timestep{0})$ is available.  
In more detail, \ac{mclstm} can be considered a \emph{linear flow} with 
the mapping $\vec{c}\timestep{t+1}= \mat{R}\timestep{t}\vec{c}\timestep{t}$ and 
$p(\vec{c}\timestep{t+1}) = p(\vec{c}\timestep{t}) |\det \mat{R}\timestep{t}|^{-1}$ in this case. The gate providing the redistribution 
matrix (see Eq.~\ref{eq:R}) is the \emph{conditioner}
in a normalizing flow model. From the perspective of normalizing flows, 
\ac{mclstm} can be considered as a flow trained in a supervised fashion. 
\citet{decoB95} proposed volume-conserving neural networks, which conserve 
the volume spanned by input vectors and thus the information of the starting point of an input is kept. 
In other words, they are constructed so that the Jacobians of the mapping from one layer 
to the next have a determinant of 1.
In contrast, the determinant of the Jacobians in \acp{mclstm} is generally smaller than $1$ (except for  degenerate cases), which 
means that volume of the inputs is not conserved. 

\paragraph{Relation to \acl{lrp}.} 
\ac{lrp} \citep{bach2015pixel} is similar to our approach
with respect to the idea that the sum of a quantity, 
the relevance $\boldsymbol{Q}\timestep{l}$ is conserved over layers $l$. \ac{lrp} aims to maintain the sum of the relevance values
$\sum_{k=1}^K Q_k\timestep{l-1}=\sum_{k=1}^K Q_k\timestep{l}$ 
backward through a classifier in order to a obtain relevance values 
for each input feature.

\paragraph{Relation to other networks that conserve particular properties.}
While a standard feed-forward neural network does not give guarantees aside from 
the conservation of the proximity of datapoints through the continuity property.
The \emph{conservation of the first moments of the data distribution} in the form
of normalization techniques \citep{ioffe2015batch, ba16layer} has had tremendous success. 
Here, batch
normalization  \citep{ioffe2015batch} could exactly conserve mean and variance across layers, whereas self-normalization \citep{klambauer2017self}
conserves those approximately. The \emph{conservation of the spectral norm}
of each layer in the forward pass has enabled the stable training of generative
adversarial networks \citep{miyato2018spectral}.
The \emph{conservation of the spectral norm of the errors} through 
the backward pass of \acp{rnn} has enabled the avoidance of the vanishing gradient problem \citep{hochreiter91, hochreiterS97}.
In this work, we explore an architecture that exactly 
\emph{conserves the mass of a subset of the input}, where mass is 
defined as a physical quantity such as mass or energy. 

Similarly, unitary \acp{rnn} \citep{arjovsky16unitary, wisdom16capacity, jing17tunable, helfrich18cayley} have been used to resolve the vanishing gradient problem.
By using unitary weight matrices, the $L_2$ norm is preserved in both the forward and backward pass.
On the other hand, the redistribution matrix in \acp{mclstm} assures that the $L_1$ norm is preserved in the forward pass.

\paragraph{Relation to geometric deep learning.}
The field of \ac{gdl} aims to provide a unification of inductive biases in representation learning \citep{bronstein21geometric}.
The main tool for this unification is symmetry, which can be expressed in terms of invarant and equivariant functions.
From the perspective of \ac{gdl}, \ac{mclstm} implements an equivariant mapping on the mass inputs w.r.t shift and scale.

\paragraph{Relation to neural networks for physical systems.} 
Neural networks have been shown to discover physical concepts such as the conservation of energies \citep{iten2020discovering}, 
and neural networks could allow to 
learn natural laws from observations \citep{schmidt2009distilling, cranmer2020discovering}.
\ac{mclstm} can be seen as a neural network architecture with 
physical constraints \citep{karpatneAFSBGSSK17, beuclerRPG19}.
It is however also possible to impose conservation laws by using other means, e.g. initialization, constrained optimization or soft constraints \citep[as, for example, proposed by][]{karpatneAFSBGSSK17, beuclerRPG19, beuclerPGOB19, jia2019physics}. 
\acp{hnn} \citep{greydanusDY19} and Symplectic Recurrent Neural Networks \citep{chen2019symplectic} make energy conserving predictions by using the Hamiltonian, 
a function that maps the inputs to the quantity that needs to be conserved. 
By using the symplectic gradients, it is possible to move around in the input space, without changing the output of the Hamiltonian. 
Lagrangian Neural Networks \citep{cranmer2020lagrangian}, extend the Hamiltonian concept by  making it possible to use arbitrary coordinates as inputs.

All of these approaches, while very promising, assume closed physical systems and are thus too restrictive for the application we have in mind.
\citet{raissi2019physics} propose to enforce physical constraints on simple feed-forward 
networks by computing the partial derivatives with respect to the inputs and computing the 
partial differential equations explicitly with the resulting terms. This approach, while 
promising, does require an exact knowledge of the governing equations. By contrast, our 
approach is able to learn its own representation of the underlying process, while 
obeying the pre-specified conservation properties.

\section{Experiments}
\label{sec:experiments}
In the following, we demonstrate the broad applicability and high predictive performance of \ac{mclstm} in settings where mass conservation is required\footnote{Code for the experiments can be found at \url{https://github.com/ml-jku/mc-lstm}}. 
Since there is no quantity to conserve in standard benchmarks for language models, 
we use benchmarks from areas in which a quantity has to be conserved. 
We assess \ac{mclstm} on the benchmarking setting in the area of neural 
arithmetics \citep{traskHRRD18,madsenJ20,heimTV20, faberW21}, in physical modeling on the damped
pendulum modeling task by \citep{iten2020discovering}, and in environmental
modeling on flood forecasting \citep{kratzert2019universal}. Additionally, we 
demonstrate the applicability of \ac{mclstm} to a traffic forecasting setting. 
For more details on the datasets and hyperparameter selection for each 
experiment, we refer to Appendix~\ref{app:details}.

\subsection{Arithmetic Tasks}
\label{sec:experiments:arithmetic}
    
    \paragraph{Addition problem.} We first considered a problem for which exact mass conservation is required. One example for such a problem has been described in the original \ac{lstm} paper \citep{hochreiterS97}, showing that \ac{lstm} is capable of summing two arbitrarily marked elements in a sequence of random numbers. We show that \ac{mclstm} is able to solve this task, but also generalizes better to longer sequences, input values in a different range and more summands. Table~\ref{tab:lstm_addition} summarizes the results of this method comparison and shows that \ac{mclstm} significantly outperformed the other models on all tests ($p$-value $\leq 0.03$, Wilcoxon test). In Appendix~\ref{app:details:arithmetic:analysis}, we provide a qualitative analysis of the learned model behavior for this task.
    
   
    \paragraph{Recurrent arithmetic.}
    Following \citet{madsenJ20}, the inputs for this task are 
    sequences of vectors, uniformly drawn from $[1, 2]^{10}$. 
    For each vector in the sequence, the sum over two random subsets is calculated. 
    Those values are then summed over time, leading to two values. 
    The target output is obtained by applying the arithmetic operation to these two values. 
    The auxiliary input for \ac{mclstm} is 
    a sequence of ones, where the last element is $-1$ to signal the end of the sequence.
    
    We evaluated \ac{mclstm} against \acp{nau} and \acp{nac} directly in the framework of \citet{madsenJ20}. \acp{nac} and \acp{nau} use the architecture as presented in 
    \citep{madsenJ20}. 
    That is, a single hidden layer with two neurons, where the first layer is recurrent. 
    The \ac{mclstm} model has two layers, of which the second one is a fully 
    connected linear layer. 
    For subtraction an extra cell was necessary to properly discard redundant input mass.
    
    \begin{figure}
        \centering
        \includegraphics[width=0.9\linewidth]{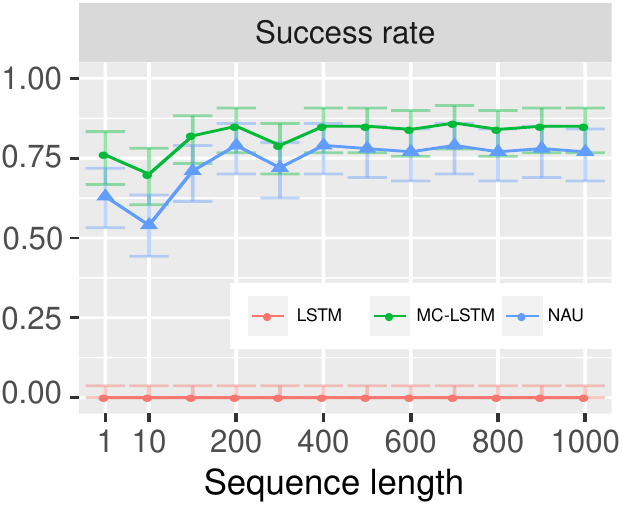}
        \caption{MNIST arithmetic task results for \ac{mclstm} and \ac{nau}. 
        The task is to correctly predict the sum of a sequence of presented
        MNIST digits. 
        The success rates are depicted on the y-axis in dependency
        of the length of the sequence (x-axis) of MNIST digits. 
        Error bars represent 95\%-confidence intervals. 
        \label{fig:seq_mnist}}
    \end{figure}
    
    For testing, the model with the lowest validation error was used, c.f. early stopping. The performance is measured by the percentage of runs that successfully 
    generalized to longer sequences. Generalization is considered successful if the error is 
    lower than the numerical imprecision of the exact operation \citep{madsenJ20}. The summary 
    in Tab.~\ref{tab:arithmeticR} shows that \ac{mclstm} was able to significantly 
    outperform the competing models ($p$-value $0.03$ for
    addition and $3\mathrm{e}{-6}$ for multiplication, proportion test). In Appendix~\ref{app:details:arithmetic:analysis}, we provide a qualitative analysis of the learned model behavior for this task.
    
     \begin{table*}
        \caption{Recurrent arithmetic task results. \acp{mclstm} for addition and subtraction/multiplication have two and three neurons, respectively. Error bars represent 95\%-confidence intervals.}
        \label{tab:arithmeticR}
        \begin{center}\begin{threeparttable}
            \begin{tabular}{rcccccc}
                \toprule
                & \multicolumn{2}{c}{addition} & \multicolumn{2}{c}{subtraction} & \multicolumn{2}{c}{multiplication} \\
                \cmidrule{2-3} \cmidrule{4-5} \cmidrule{6-7}
                & success rate\tnote{a} & updates\tnote{b} & success rate\tnote{a} & updates\tnote{b} & success rate\tnote{a} & updates\tnote{b} \\
                \midrule
                \acs{mclstm} & $\mathbf{96\%}~^{+2\%}_{-6\%}$ & $4.6 \cdot 10^5$ & 
                $\mathbf{81\%}~^{+6\%}_{-9\%}$ & $1.2 \cdot 10^5$ & $\mathbf{67\%}~^{+8\%}_{-10\%}$ & $1.8 \cdot 10^5$ \\
                \acs{lstm} & $0\%~^{+4\%}_{-0\%}$ & -- & 
                $0\%~^{+4\%}_{-0\%}$ & -- & $0\%~^{+4\%}_{-0\%}$ & -- \\
                \acs{nau} / \acs{nmu} & $88\%~^{+5\%}_{-8\%}$ & $8.1 \cdot 10^4$ &
                $60\%~^{+9\%}_{-10\%}$ & $ 6.1 \cdot 10^4$ & $34\%~^{+10\%}_{-9\%}$ & $8.5 \cdot 10^4$\\
                \acs{nac} & $56\%~^{+9\%}_{-10\%}$ & $3.2 \cdot 10^5$ & $\mathbf{86\%} {~}^{+5\%}_{-8\%}$ & $4.5 \cdot 10^4$ & $0\%~^{+4\%}_{-0\%}$ & -- \\
                \acs{nalu} & $10\% {~}^{+7\%}_{-4\%}$ & $1.0 \cdot 10^{6}$ & $0\%~^{+4\%}_{-0\%}$ & -- & $1\%~^{+4\%}_{-1\%}$ & $4.3 \cdot 10^5$ \\
                \bottomrule
            \end{tabular}
            \begin{tablenotes}
                \item [a] Percentage of runs that generalized to longer sequences.
                \item [b] Median number of updates necessary to solve the task.
            \end{tablenotes}
        \end{threeparttable}\end{center}
    \end{table*}
    
    \paragraph{Static arithmetic.} 
    To enable a direct comparison with the 
    results reported in \citet{madsenJ20}, we also compared a feed-forward variant of \ac{mclstm} on the static arithmetic task, 
    see Appendix~\ref{app:details:arithmetic:static}.

    \paragraph{MNIST arithmetic.}
    We tested that feature extractors can be learned from MNIST images 
    \citep{lecunBBH98} to perform arithmetic on the images \citep{madsenJ20}.
    This is especially of interest if mass inputs are not given directly, but can be extracted from the available data.
    The input is a sequence of MNIST images and the target output is the corresponding sum of the labels. 
    Auxiliary inputs are all $1$, except the last entry, which is $-1$, to indicate the end of the sequence.
    The models are the same as in the recurrent arithmetic task with a \ac{cnn} to 
    convert the images to (mass) inputs for these networks.
    The network is learned end-to-end. 
    $L_2$-regularization is added to the output of the \ac{cnn} to prevent its outputs from growing arbitrarily large.
    The results for this experiment are depicted in Fig.~\ref{fig:seq_mnist}. \ac{mclstm} significantly outperforms the state-of-the-art, \ac{nau} ($p$-value $0.002$, Binomial test).
    
    \begin{figure}
        \centering
        \includegraphics[width=0.9\linewidth]{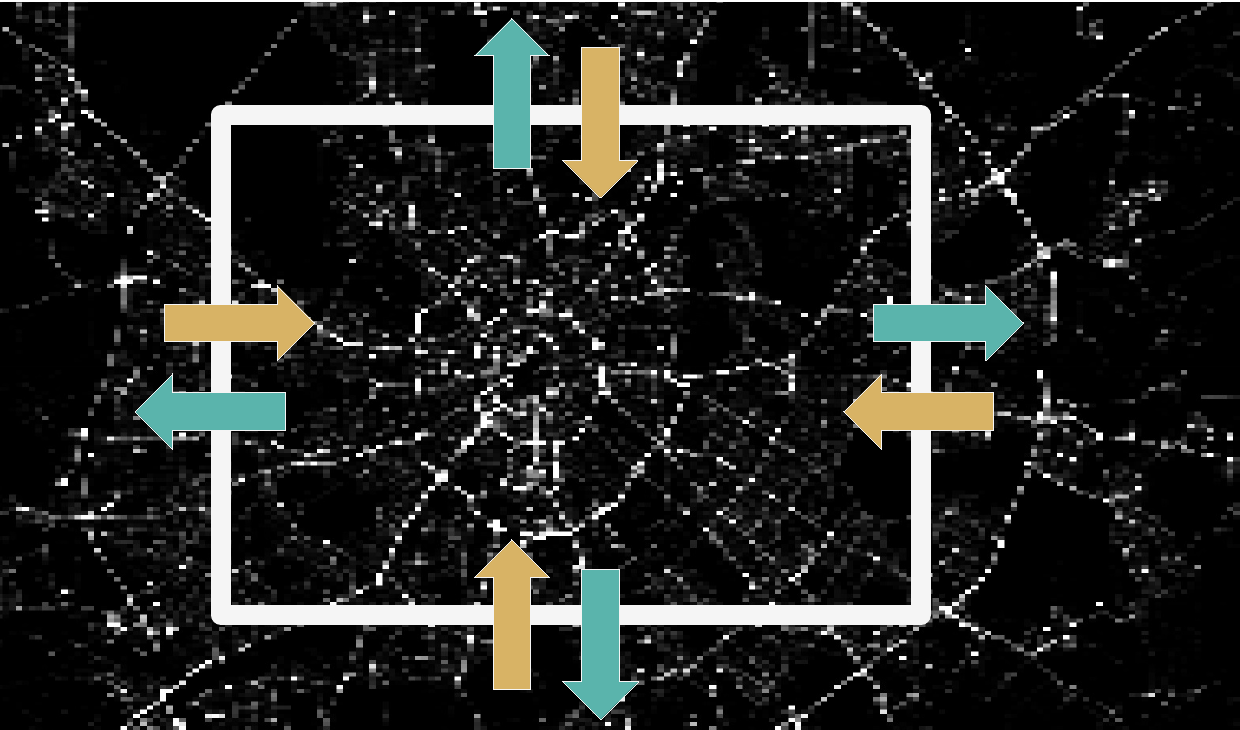}
        \caption{Schematic depiction of inbound-outbound traffic situations that 
        require the conservation-of-vehicles principle. All vehicles on outbound
        roads (yellow arrows) must have entered the city center before (green arrows) 
        or have been present in the first timestep.}
        \label{fig:traffic_inbound_outbound}
    \end{figure} 

\subsection{Inbound-outbound Traffic Forecasting}
We examined the usage of \acp{mclstm} for traffic forecasting in situations in which inbound 
and outbound traffic counts of a city are available (see Fig.~\ref{fig:traffic_inbound_outbound}). 
For this type of data, a \emph{conservation-of-vehicles} principle \citep{nam1996traffic}
must hold, since vehicles can only leave the city if they have entered it before or had been there in the first place. 
Based on data from the traffic4cast 
2020 challenge \citep{kreil2020surprising},
we constructed a 
dataset to model inbound and outbound traffic in 
three different cities: Berlin, Istanbul and Moscow. 
We compared \ac{mclstm} against \ac{lstm}, which is the state-of-the-art method for
several types of traffic forecasting situations \citep{zhao2017lstm,tedjopurnomo2020survey},
and found that \ac{mclstm} significantly outperforms \ac{lstm} in this traffic 
forecasting setting (all $p$-values $\ \leq 0.01$, Wilcoxon test). For details, see Appendix~\ref{app:addons:traffic}.

\subsection{Damped Pendulum}\label{sec:pedulum}
In the area of physics, we examined the usability of \ac{mclstm} for the problem of modeling a swinging damped pendulum. 
Here, the total energy is the conserved property.
During the movement of the pendulum, kinetic energy is converted into potential energy and vice-versa. This conversion between both energies has to be learned by the off-diagonal values of the redistribution matrix. A qualitative analysis of a trained \ac{mclstm} for this problem can be found in Appendix~\ref{app:Qualitative_analysis_pendulum}. 

\begin{figure}
    \centering
    \includegraphics[width=0.9\linewidth]{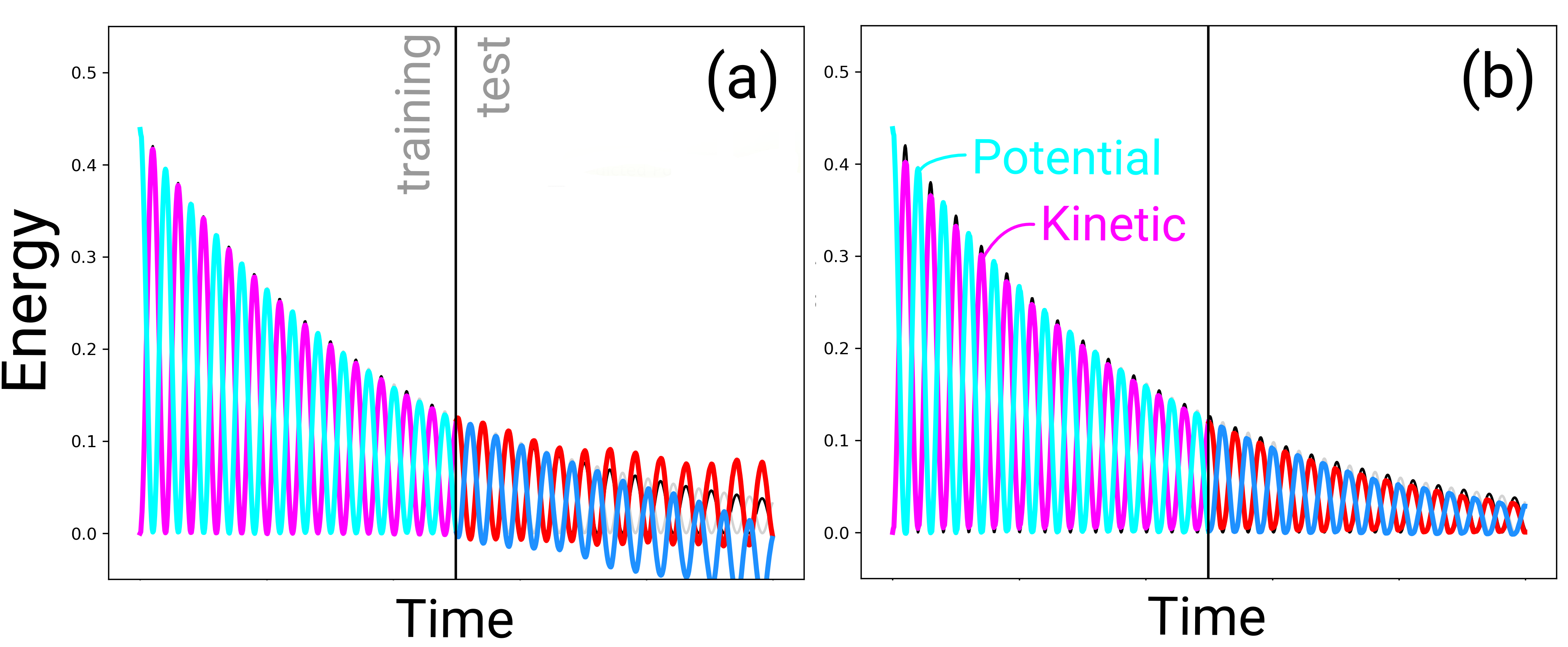}
    \caption{Example for the pendulum-modelling exercise. \textbf{(a)} \ac{lstm} trained for predicting energies of the pendulum with friction in autoregressive fashion, \textbf{(b)} \ac{mclstm} trained in the same setting. Each subplot shows the potential- and kinetic energy and the respective predictions. 
    }
    \label{fig:pendulum}
\end{figure} 

Accounting for friction, energy dissipates and 
the swinging slows over time, toward a fixed point. 
This type of behavior presents a difficulty for machine learning and is impossible for 
methods that assume the pendulum to be a closed system, such as \acsp{hnn} \citep{greydanusDY19} (see Appendix~\ref{app:sec:hnn_comparison}).
We generated $120$ datasets with timeseries of a pendulum, 
where we used multiple different settings for initial angle, length of the pendulum,
and the amount of friction. We then selected \ac{lstm} and \ac{mclstm} models and compared
them with respect to the analytical solution in terms of \ac{mse}. 
For an example, see Fig.~\ref{fig:pendulum}. 
Overall, \ac{mclstm} significantly outperformed \ac{lstm} with 
a mean \ac{mse} of $0.01$ (standard deviation $0.02$) compared to $0.07$ (standard deviation $0.14$; with a $p$-value $4.7\mathrm{e}{-10}$, Wilcoxon test).
In the friction-free case, no significant difference to \acp{hnn} was found 
(see Appendix~\ref{app:sec:hnn_comparison}).


\subsection{Hydrology: Rainfall Runoff Modeling}
\label{sec:hydrology}

\begin{table*}
\caption{Hydrology benchmark results. All values represent the median (25\% and 75\% percentile in sub- and superscript, respectively) over the 447 basins.}
\label{tab:hydrology}
\begin{center}\begin{threeparttable}
\begin{tabular}{lccccc}
\toprule
{} &         MC\textsuperscript{a} &       NSE\textsuperscript{b} &  $\beta$-NSE\textsuperscript{c} &        FLV\textsuperscript{d} &        FHV\textsuperscript{e} \\
\midrule
MC-LSTM Ensemble   &  \ding{51} &  \textit{0.744}$\,_{0.641}^{0.814}$ & -0.020$\,_{-0.066}^{0.013}$ & -24.7$\,_{-94.4}^{31.1}$ & \textbf{-14.7}$\,_{-23.4}^{-7.0}$ \\
LSTM Ensemble      &  \ding{55} &  \textbf{0.763}$\,_{0.676}^{0.835}$ & -0.034$\,_{-0.077}^{-0.002}$ &  36.3$\,_{-0.4}^{59.7}$ & \textit{-15.7}$\,_{-23.8}^{-8.6}$ \\
SAC-SMA        &  \ding{51} &  0.603$\,_{0.512}^{0.682}$ & -0.066$\,_{-0.108}^{-0.026}$ &  37.4$\,_{-31.9}^{68.1}$ & -20.4$\,_{-29.9}^{-12.2}$ \\
VIC (basin)    &  \ding{51} &  0.551$\,_{0.465}^{0.641}$ & \textit{-0.018}$\,_{-0.071}^{0.032}$ & -74.8$\,_{-271.8}^{23.1}$ & -28.1$\,_{-40.1}^{-17.5}$ \\
VIC (regional) &  \ding{51} &  0.307$\,_{0.218}^{0.402}$ & -0.074$\,_{-0.166}^{0.023}$ &  18.9$\,_{-73.1}^{69.6}$ & -56.5$\,_{-64.6}^{-38.3}$ \\
mHM (basin)    &  \ding{51} &  0.666$\,_{0.588}^{0.730}$ & -0.040$\,_{-0.102}^{0.003}$ &  \textit{11.4}$\,_{-64.0}^{65.1}$ & 
-18.6$\,_{-27.7}^{-9.5}$ \\
mHM (regional) &  \ding{51} &  0.527$\,_{0.391}^{0.619}$ & -0.039$\,_{-0.169}^{0.033}$ &  36.8$\,_{-32.6}^{70.9}$ & -40.2$\,_{-51.0}^{-23.8}$ \\
HBV (lower)    &  \ding{51} &  0.417$\,_{0.276}^{0.550}$ & -0.023$\,_{-0.114}^{0.058}$ &  23.9$\,_{-25.9}^{61.0}$ & -41.9$\,_{-55.2}^{-17.3}$ \\
HBV (upper)    &  \ding{51} &  0.676$\,_{0.578}^{0.749}$ & \textbf{-0.012}$\,_{-0.058}^{0.034}$ &  18.3$\,_{-62.9}^{67.5}$ & -18.5$\,_{-27.8}^{-8.5}$ \\
FUSE (900)     &  \ding{51} &  0.639$\,_{0.539}^{0.715}$ & -0.031$\,_{-0.100}^{0.024}$ & \textbf{-10.5}$\,_{-94.8}^{49.2}$ & -18.9$\,_{-27.8}^{-9.9}$ \\
FUSE (902)     &  \ding{51} &  0.650$\,_{0.570}^{0.727}$ & -0.047$\,_{-0.098}^{-0.004}$ & -68.2$\,_{-239.9}^{17.1}$ & -19.4$\,_{-27.9}^{-8.9}$\\
FUSE (904)     &  \ding{51} &  0.622$\,_{0.527}^{0.705}$ & -0.067$\,_{-0.135}^{-0.019}$ & -67.6$\,_{-238.6}^{35.7}$ & -21.4$\,_{-33.0}^{-11.3}$ \\
\bottomrule
\end{tabular}
\begin{tablenotes}[para]
\small{
\textsuperscript{a}: \textit{Mass conservation (MC)}. \\
\textsuperscript{b}: \textit{Nash-Sutcliffe efficiency: $(-\infty, 1]$, values closer to one are desirable.} \\
\textsuperscript{c}: \textit{$\beta$-NSE decomposition: $(-\infty, \infty)$, values closer to zero are desirable.}\\
\textsuperscript{d}: \textit{Bottom 30\% low flow bias: $(-\infty, \infty)$, values closer to zero are desirable.}\\
\textsuperscript{e}: \textit{Top 2\% peak flow bias: $(-\infty, \infty)$, values closer to zero are desirable.}\\
}
\end{tablenotes}
\end{threeparttable}\end{center}
\end{table*}

We tested \ac{mclstm} for large-sample hydrological modeling following \citet{kratzert2019universal}.
An ensemble of 10 \acp{mclstm} was trained on 10 years of data from 447 basins using the publicly-available CAMELS dataset \citep{newman2015development, addor2017camels}.
The mass input is precipitation and auxiliary inputs are: daily min. and max. temperature, solar radiation, and vapor pressure, plus 27 basin characteristics related to geology, vegetation, and climate \citep[described by][]{kratzert2019universal}. 
All models, apart from \ac{mclstm} and \ac{lstm}, were trained 
 by different research groups with experience using each model.
More details are given in Appendix~\ref{app:hyd-training-setup}.


As shown in Tab.~\ref{tab:hydrology}, \ac{mclstm} performed better with respect to the Nash--Sutcliffe Efficiency (NSE; the $R^2$ between simulated and observed runoff) 
than any other mass-conserving hydrology model, although slightly worse than \ac{lstm}. 

NSE is often not the most important metric in hydrology, since water managers are typically concerned 
primarily with extremes (e.g.\ floods). 
\ac{mclstm} performed significantly better ($p = 0.025$, Wilcoxon test) than 
all models, including \ac{lstm}, with respect to high volume flows (FHV), at or above the 
98th percentile flow in each basin.
This makes \ac{mclstm} the current state-of-the-art model for flood prediction. 
\ac{mclstm} also performed significantly better than \ac{lstm} on low volume 
flows (FLV) and overall bias, however there are other hydrology models that are 
better for predicting low flows (which is 
important, e.g.\ for managing droughts).

\paragraph{Model states and environmental processes.}
It is an open challenge to bridge the gap between the fact that \ac{lstm} approaches give generally better 
predictions than other models (especially for flood prediction) and the fact that water managers need 
predictions that help them understand not only how much water will be in a river at a given time, but also how water moves through a basin. 

\begin{figure}[htpb]
\centering
\includegraphics[width=\columnwidth, trim=0 10 0 0, clip]{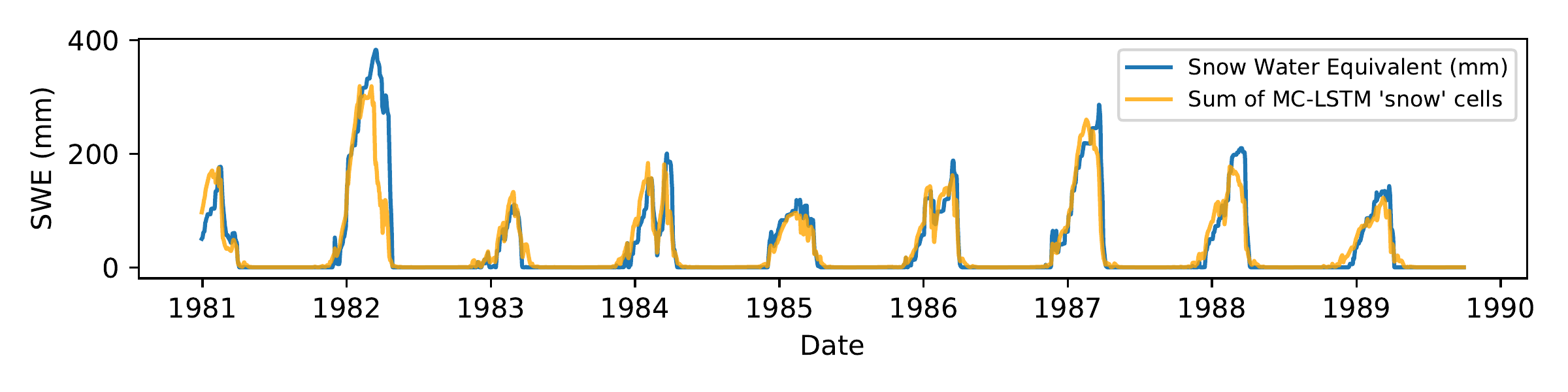}
\caption{Snow-water-equivalent (SWE) from a single basin. The blue line is SWE modeled by \citet{newman2015development}. The orange line is the sum over 4 \ac{mclstm} memory cells (Pearson correlation coefficient $r \geq 0.8$).}

\label{fig:snow-cell}
\end{figure}

Snow processes are difficult to observe and model.
\citet{kratzert2019neuralhydrology} showed that \ac{lstm} learns to track snow in memory cells without 
requiring snow data for training. 
We found similar behavior in \acp{mclstm}, which has the advantage of doing this with memory cells that are 
\emph{true} mass storages. 
Figure \ref{fig:snow-cell} shows the snow as the sum over a subset of \ac{mclstm} memory states and snow water 
equivalent (SWE) modeled by the well-established Snow-17 snow model \citep{anderson1973national} (Pearson correlation coefficient $r \geq 0.91$). 
It is important to note that \acp{mclstm} did not have access to 
any snow data during training.  
In the best case, it is possible to take advantage 
of the inductive bias to predict how much water will be stored 
as snow under different conditions by using simple combinations or 
mixtures of the internal states. 
Future work will determine whether this is possible with other difficult-to-observe states and fluxes.

\subsection{Ablation Study} 
In order to demonstrate that the design choices of \ac{mclstm} 
are necessary together to enable accurate predictive models, 
we performed an ablation study. In this study, we made changes that 
disrupt the mass conservation property a) of the input gate, 
b) the redistribution operation, and c) the output gate. 
We tested these three variants on data from the hydrology experiments. 
We chose 5 random basins to limit computational expenses and 
trained nine repetitions for 
each configuration and basin. 
The
strongest decrease in performance is observed if the
redistribution matrix
does not conserve mass, and smaller decreases 
if input or output gate do not conserve mass. 
The results of the ablation study indicate that the design of the input gate, 
redistribution matrix, and output gate, are necessary together 
to obtain accurate and mass-conserving models (see Appendix Tab.~\ref{tab:ablation}). 

\section{Conclusion}
\label{sec:outro}
We have demonstrated how to design an \ac{rnn} that has the property to conserve mass of particular inputs. 
This architecture is proficient as neural arithmetic unit and is well-suited for predicting physical systems like hydrological processes, in which water mass has to be conserved. 
We envision that \ac{mclstm} can become a powerful tool in modeling environmental, sustainability, and biogeochemical cycles.

\section*{Acknowledgments}
The ELLIS Unit Linz, the LIT AI Lab, the Institute for Machine Learning, are supported by the Federal State Upper Austria. 
IARAI is supported by Here Technologies.
We thank the projects AI-MOTION (LIT-2018-6-YOU-212), DeepToxGen (LIT-2017-3-YOU-003), AI-SNN (LIT-2018-6-YOU-214), DeepFlood (LIT-2019-8-YOU-213), Medical Cognitive Computing Center (MC3), PRIMAL (FFG-873979), S3AI (FFG-872172), DL for granular flow (FFG-871302), ELISE (H2020-ICT-2019-3 ID: 951847), AIDD (MSCA-ITN-2020 ID: 956832).
We thank Janssen Pharmaceutica, UCB Biopharma SRL, Merck Healthcare KGaA, Audi.JKU Deep Learning Center, TGW LOGISTICS GROUP GMBH, Silicon Austria Labs (SAL), FILL Gesellschaft mbH, Anyline GmbH, Google, ZF Friedrichshafen AG, Robert Bosch GmbH, Software Competence Center Hagenberg GmbH, T\"{U}V Austria, and the NVIDIA Corporation.

\bibliography{references}
\bibliographystyle{icml2021}

\clearpage


\renewcommand\thefigure{\thesection.\arabic{figure}}    
\setcounter{figure}{0}    

\renewcommand\thetable{\thesection.\arabic{table}}    
\setcounter{table}{0}

\icmltitlerunning{\acs{mclstm}: Appendix}

\appendix

\section{Notation Overview}
\label{app:notation}

    Most of the notation used throughout the paper, is summarized in Tab.~\ref{tab:notation}.

    \begin{table*}
        \caption{Symbols and notations used in this paper.}
        \label{tab:notation}
        \begin{center}\begin{tabular}{lcl} 
            \toprule
            Definition &  Symbol/Notation & Dimension \\ 
            \midrule
             mass input at timestep $t$   & $\vec{x}\timestep{t}$ or $\mass{x}\timestep{t}$ & $M$ or $1$ \\
             auxiliary input at timestep $t$   & $\aux{\vec{a}}\timestep{t}$&  $L$ \\
             cell state at timestep $t$   & $\vec{c}\timestep{t}$& $K$ \\
            limit of sequence of cell states & $\vec{c}^\infty$ &  \\
             hidden state at timestep $t$   & $\vec{h}\timestep{t}$& $K$ \\
             redistribution matrix & $\mat{R}$ & $K \times K$ \\
             input gate & $\vec{i}$ & $K$ \\
             output gate & $\vec{o}$ & $K$ \\
             mass & $\vec{m}$ & $K$ \\
             input gate weight matrix & $\mat{W}_\mathrm{i}$ & $K \times L$ \\
             input gate weight matrix & $\mat{W}_\mathrm{o}$ & $K \times L$ \\
             output gate weight matrix& $\mat{U}_\mathrm{i}$ & $K \times K$ \\
             output gate weight matrix & $\mat{U}_\mathrm{o}$ & $K \times K$ \\
             identity matrix & $\mat{K}$ & $K \times K$ \\
            
             input gate bias & $\vec{b}_\mathrm{i}$ & $K$ \\
             output gate bias & $\vec{b}_\mathrm{o}$ & $K$ \\
             arbitrary differentiable function & $f$ & \\
             hypernetwork function (conditioner)  & $g$ & \\
             redistribution gate bias & $\mat{B}_\mathrm{R}$ & $K \times K$ \\
            
             stored mass  & $m_c$ &  \\
             mass efflux & $m_h$ &  \\
             limit of series of mass inputs & $m_x^\infty$ &  \\
            
             timestep index & $t$  &  \\
             an arbitrary timestep & $\tau$  &  \\
             last timestep of a sequence & $T$  &  \\
               
             redistribution gate weight tensor & $\ten{W}_\mathrm{r}$ & $K \times K \times L$  \\
             redistribution gate weight tensor & $\ten{U}_\mathrm{r}$ & $K \times K \times K$  \\
              
             arbitrary feature index  & $a$  &  \\
             arbitrary feature index  & $b$  &  \\
             arbitrary feature index & $c$  &  \\
             
             \bottomrule
        \end{tabular}\end{center}
    \end{table*}

\section{Experimental Details}
\label{app:details}

    In the following, we provide further details on the experimental setups.

\subsection{Neural Arithmetic}
\label{app:details:arithmetic}

Neural networks that learn arithmetic operations have recently 
come into focus \citep{traskHRRD18,madsenJ20}. 
Specialized neural modules for arithmetic operations could play a role
for complex AI systems since cognitive studies indicate that 
there is a part of the brain that enables animals and humans 
to perform basic arithmetic operations \citep{nieder16, gallistel18}. 
Although this primitive number processor can only perform approximate 
arithmetic, it is a fundamental part of our ability to understand and 
interpret numbers \citep{dehaene11}.

\subsubsection{Details on Datasets}
\label{app:details:arithmetic:data}
    
    We consider the \emph{addition problem} that was proposed in the original \ac{lstm} paper \citep{hochreiterS97}. We chose input values in the range $[0, 0.5]$ in order to be able to use the fast standard implementations of \ac{lstm}. For this task, 20\,000 samples were generated using a fixed random seed to create a dataset, which was split in 50\% training and 50\% validation samples. For the test data, 1\,000 samples were generated with a different random seed.
    
    A definition of the \emph{static arithmetic} task is provided by \citep{madsenJ20}. The following presents this definition and its extension to the \emph{recurrent arithmetic} 
    task \citep[c.f.][]{traskHRRD18}. 
    
    The input for the static version is a vector, $\vec{x} \in \mathcal{U}(1, 2)^{100}$, consisting of numbers that are drawn randomly from a uniform distribution. The target, $y$, is computed as
    \begin{equation*}
        y = \left(\sum_{k=a}^{a + c} x_k\right) \mathbin{\square} \left(\sum_{k=b}^{b + c} x_k\right),
    \end{equation*}
    where $c \in \nats$, $a \leq b \leq a + c \in \nats$ and $\square \in \{ +, -, \cdot \}$. For the recurrent variant , the input consists of a sequence of $T$ vectors, denoted by $\vec{x}\timestep{t} \in \mathcal{U}(1, 2)^{10}, t \in \{1, \ldots, T\}$, and the labels are computed as
    \begin{equation*}
        y = \left(\sum_{t=1}^T \sum_{k=a}^{a + c} x_k\timestep{t}\right) \mathbin{\square} \left(\sum_{t=1}^T \sum_{k=b}^{b + c} x_k\timestep{t}\right).
    \end{equation*}
    For these experiments, no fixed datasets were used. Instead, samples were generated on the fly. For the recurrent tasks, 2\,000\,000 batches of 128 problems were created and for the static tasks 500\,000 batches of 128 samples were used in the addition and subtraction tasks and 3\,000\,000 batches for multiplication. Note that since the subsets overlap, i.e., inputs are re-used, this data does not have mass conservation properties. 
    
    For a more detailed description of the \emph{MNIST addition} data, we refer to \citep{traskHRRD18} and the appendix of \citep{madsenJ20}.
    
\subsubsection{Details on Hyperparameters.}
\label{app:details:arithmetic:hyper}

    For the \emph{addition problem}, every network had a single hidden layer with 10 units. The output layer was a linear, fully connected layer for all \ac{mclstm} and \ac{lstm} variants. The \ac{nau} \citep{madsenJ20} and \ac{nalu}/\ac{nac} \citep{traskHRRD18} networks used their corresponding output layer. Also, we used a more common $L_2$ regularization scheme with low regularization constant ($10^{-4}$) to keep the weights ternary for the \ac{nau}, rather than the strategy used in the reference implementation from \citet{madsenJ20}. Optimization was done using Adam \citep{kingmaB15} for all models. The initial learning rate was selected from $\{0.1, 0.05, 0.01, 0.005, 0.001\}$ on the validation data for each method individually. All methods were trained for 100 epochs.

    The weight matrices of \ac{lstm} were initialized in a standard way, using orthogonal and identity matrices for the forward and recurrent weights, respectively. Biases were initialized to be zero, except for the bias in the forget gate, which was initialized to 3. This should benefit the gradient flow for the first updates. Similarly, \ac{mclstm} is initialized so that the redistribution matrix (cf. Eq.~\ref{eq:mclstm:redistribution}) is (close to) the identity matrix. Otherwise we used orthogonal initialization \citep{saxe2013exact}. The bias for the output gate was initialized to -3. This stimulates the output gates to stay closed (keep mass in the system), which has a similar effect as setting the forget gate bias in \ac{lstm}. This practically holds for all subsequently described experiments.
    
    For the \emph{recurrent arithmetic tasks}, we tried to stay as close as possible to the setup that was used by \citet{madsenJ20}. This means that all networks had again a single hidden layer. The \ac{nau}, \ac{nmu} and \ac{nalu} networks all had two hidden units and, respectively, \ac{nau}, \ac{nmu} and \ac{nalu} output layers. The first, recurrent layer for the first two networks was a \ac{nau} and the \ac{nalu} network used a recurrent \ac{nalu} layer.
    For the exact initialization of \ac{nau} and \ac{nalu}, we refer to \citep{madsenJ20}.
    
    The \ac{mclstm} models used a fully connected linear layer with $L_2$-regularization for projecting the hidden state to the output prediction for the addition and subtraction tasks. A free linear layer was used to compensate for the fact that the data does not have mass-conserving properties. However, it is important to note that the mass conservation in \ac{mclstm} is still necessary to solve this task. For the multiplication problem, we used a multiplicative, non-recurrent variant of \ac{mclstm} with an extra scalar parameter to allow the conserved mass to be re-scaled if necessary. This multiplicative layer is described in more detail in Appendix~\ref{app:details:arithmetic:static}. 
    
    Whereas the addition could be solved with two hidden units, \ac{mclstm} needed three hidden units to solve both subtraction and multiplication. This extra unit, which we refer to as the \emph{trash cell}, allows \acp{mclstm} to get rid of excessive mass that should not influence the prediction. Note that, since the mass inputs are vectors, the input gate has to be computed in a similar fashion as the redistribution matrix. Adam was again used for the optimization. We used the same learning rate ($0.001$) as \citet{madsenJ20} to train the \ac{nau}, \ac{nmu} and \ac{nalu} networks. For \ac{mclstm} the learning rate was increased to 0.01 for addition and subtraction and 0.05 for multiplication after a manual search on the validation set. All models were trained for two million update steps.
    
    In a similar fashion, we used the same models from \citet{madsenJ20} for the \emph{MNIST addition} task. For \ac{mclstm}, we replaced the recurrent \ac{nau} layer with a \ac{mclstm} layer and the output layer was replaced with a fully connected linear layer. In this scenario, increasing the learning rate was not necessary. This can probably be explained by the fact that training \ac{cnn} to regress the MNIST images is the main challenge during learning. We also used a standard $L_2$-regularization on the outputs of \ac{cnn} instead of the implementation proposed in \citep{madsenJ20} for this task.

\subsubsection{Static Arithmetic}
\label{app:details:arithmetic:static}
    This experiment should enable a more direct comparison to the results from \citet{madsenJ20} than the recurrent variant. The data for the static task is equivalent to that of the recurrent task with sequence length one. For more details on the data, we refer to Appendix~\ref{app:details:arithmetic:data} or \citep{madsenJ20}.
    
    Since the static task does not require a recurrent model, we discarded the redistribution matrix in \ac{mclstm}. The result is a layer with only input and output gates, which we refer to as a \ac{mcfc} layer. We compared this model to the results reported in \citep{madsenJ20}, using the code base that accompanied the paper. All \ac{nalu} and \ac{nau} networks had a single hidden layer. Similar to the recurrent task, \ac{mclstm} required two hidden units for addition and three for subtraction.
    Mathematically, an \ac{mcfc} with $K$ hidden neurons and $M$ inputs can be defined as $\mathrm{\ac{mcfc}} : \reals^M \to \reals^K : \vec{x} \mapsto \vec{y}$, where
    \begin{align*}
        \vec{y} & = \diag(\vec{o}) \cdot \mat{I} \cdot \vec{x} &
        \mat{I} & = \softmax(\mat{B}_I) &
        \vec{o} & = \sigmoid(\vec{b}_o),
    \end{align*}
    where the softmax operates on the row dimension to get a column-normalized matrix, $\mat{I}$, for the input gate. 
    
    Using the $\log$-$\exp$ transform \citep[c.f.][]{traskHRRD18}, a multiplicative \ac{mcfc} with scaling parameter, $\vec{\alpha}$, can be constructed as follows:
    $\exp(\mathrm{\ac{mcfc}}(\log(\vec{x})) + \vec{\alpha})$.
    The scaling parameter is necessary to break the mass conservation when it is not needed. By replacing the output layer with this multiplicative \ac{mcfc}, it can also be used to solve the multiplication problem. This network also required three hidden neurons. This model was compared to a \ac{nmu} network with two hidden neurons and \ac{nalu} network. 
    
    All models were trained for two million updates with the Adam optimizer \citep{kingmaB15}. The learning rate was set to 0.001 for all networks, except for the \ac{mcfc} network, which needed a lower learning rate of 0.0001, and the multiplicative \ac{mcfc} variant, which was trained with learning rate 0.01. These hyperparameters were found using a manual search.
    
    \begin{table*}
        \caption{Results for the static arithmetic task. \ac{mcfc} is a mass-conserving variant of \ac{mclstm} based on fully-connected layers for non-recurrent tasks. \acp{mcfc} for addition and subtraction/multiplication have two and three neurons, respectively. 
        Error bars represent 95\% confidence intervals.}
        \label{tab:arithmeticS}
        \begin{center}\begin{threeparttable}
            \begin{tabular}{rcccccc}
                \toprule
                & \multicolumn{2}{c}{addition} & \multicolumn{2}{c}{subtraction} & \multicolumn{2}{c}{multiplication} \\
                \cmidrule{2-3} \cmidrule{4-5} \cmidrule{6-7}
                 & success rate\tnote{a} & updates\tnote{b} & success rate\tnote{a} & updates\tnote{b} & success rate\tnote{a} & updates\tnote{b} \\
                \midrule
                \acs{mcfc} & $\mathbf{100\%}~^{+0\%}_{-4\%}$ & $2.1 \cdot 10^5$ & 
                $\mathbf{100\%}~^{+0\%}_{-4\%}$ & $1.6 \cdot 10^5$ &
                $\mathbf{100\%}~^{+0\%}_{-4\%}$ & $1.4 \cdot 10^6$ \\
                \acs{nau} / \acs{nmu} & $\mathbf{100\%}~^{+0\%}_{-4\%}$ & $1.8 \cdot 10^4$ &
                $\mathbf{100\%}~^{+0\%}_{-4\%}$ & $ 5.0 \cdot 10^3$ &
                $98\%~^{+1\%}_{-5\%}$ & $1.4 \cdot 10^6$ \\
                \acs{nac} & $\mathbf{100\%}~^{+0\%}_{-4\%}$ & $2.5 \cdot 10^5$ &
                $\mathbf{100\%}~^{+0\%}_{-4\%}$ & $9.0 \cdot 10^3$ &
                $31\%~^{+10\%}_{-8\%}$ & $2.8 \cdot 10^6$ \\
                \acs{nalu} & $14\%~^{+8\%}_{-5\%}$ & $1.5 \cdot 10^6$ &
                $14\%~^{+8\%}_{-5\%}$ & $1.9 \cdot 10^6$ &
                $0\%~^{+4\%}_{-0\%}$ & -- \\
                \bottomrule
            \end{tabular}
            \begin{tablenotes}
                \item [a] Percentage of runs that generalized to a different input range.
                \item [b] Median number of updates necessary to solve the task.
            \end{tablenotes}
        \end{threeparttable}\end{center}
    \end{table*}
    
    Since the input consists of a vector, the input gate predicts a left-stochastic matrix, similar to the redistribution matrix. This allows us to verify generalization abilities of the inductive bias in \acp{mclstm}. The performance was measured in a similar way as for the recurrent task, except that generalization was tested over the range of the input values \citep{madsenJ20}. Concretely, the models were trained on input values in 
    $[1, 2]$ and tested on input values in the range $[2, 6]$. Table~\ref{tab:arithmeticS} shows that \ac{mcfc} is able to match or outperform both \ac{nalu} and \ac{nau} on this task.
    

    
\subsubsection{Comparison with Time-dependent \ac{mclstm}}

    We used \ac{mclstm} with a time-independent redistribution matrix, as in Eq.~\eqref{eq:mclstm:redistribution}, to solve the addition problem. This resembles another form of inductive bias, since we know that no redistribution across cells is necessary to solve this problem and it results also in a more efficient model, because less parameters have to be learned. However, for the sake of flexibility, we also verified that it is possible to use the more general time-dependent redistribution matrix (cf. Eq.~\ref{eq:R}). The results of this experiment can be found in Table~\ref{tab:lstm_addition_full}.
    
    Although the performance of \ac{mclstm} with time-dependent redistribution matrix is slightly worse than that of the more efficient \ac{mclstm} variant, it still outperforms all other models on the generalisation tasks. This can partly be explained by the fact that is harder to train a time-dependent redistribution matrix, while the training budget is limited to 100 epochs.
    
    \begin{table*}
        \caption{Performance of different models on the \ac{lstm} addition task in terms of the \ac{mse}. \ac{mclstm} significantly (all $p$-values below $.05$) outperforms its competitors, \ac{lstm} (with high initial forget gate bias), \ac{nalu}, \ac{nau} a layer-normalized \ac{lstm} (LN-\acs{lstm}) and unitary \acp{rnn} (U\acs{rnn}). Error bars represent 95\%-confidence intervals across 100 runs.}
        \label{tab:lstm_addition_full}
        
        \begin{center}\begin{threeparttable}
            \begin{tabular}{lr@{$\ \pm\ $}lr@{$\ \pm\ $}lr@{$\ \pm\ $}lr@{$\ \pm\ $}lr@{$\ \pm\ $}lr}
                \toprule
                & \multicolumn{2}{c}{reference\tnote{a}} & \multicolumn{2}{c}{seq length\tnote{b}} & \multicolumn{2}{c}{input range\tnote{c}} & \multicolumn{2}{c}{count\tnote{d}} & \multicolumn{2}{c}{combo\tnote{e}} & \texttt{NaN}\tnote{f} \\
                \midrule
                \acs{mclstm}\tnote{\textdagger} & 0.013 & 0.004 & 0.022 & 0.010 & 2.6 & 0.8 & 2.2 & 0.7 & 13.6 & 4.0 & 0 \\
                \acs{mclstm} & \textbf{0.004} & 0.003 & \textbf{0.009} & 0.004 & \textbf{0.8} & 0.5 & \textbf{0.6} & 0.4 & \textbf{4.0} & 2.5 & 0 \\
                \acs{lstm} & 0.008 & 0.003 & 0.727 & 0.169 & 21.4 & 0.6 & 9.5 & 0.6 & 54.6 & 1.0 & 0 \\
                LN-\acs{lstm} & 0.026 & 0.003 & 0.055 & 0.010 & 24.5 & 0.3 & 7.5 & 0.2 & 62.0 & 0.5 & 0 \\
                U\acs{rnn} & 0.043 & 0.001 & 0.139 & 0.133 & 99.9 & 63.8 & 7.0 & 0.1 & 88.1 & 3.4 & 0 \\
                \acs{nalu} & 0.060 & 0.008 & 0.059 & 0.009 & 25.3 & 0.2 & 7.4 & 0.1 & 63.7 & 0.6 & 93 \\
                \acs{nau} & 0.248 & 0.019 & 0.252 & 0.020 & 28.3 & 0.5 & 9.1 & 0.2 & 68.5 & 0.8 & 24 \\
                \bottomrule
            \end{tabular}
            \begin{tablenotes}
                \item [a] training regime: \hfill summing 2 out of 100 numbers between 0 and 0.5.
                \item [b] longer sequence lengths: \hfill summing 2 out of 1\,000 numbers between 0 and 0.5.
                \item [c] more \emph{mass} in the input: \hfill summing 2 out of 100 numbers between 0 and 5.0.
                \item [d] higher number of summands: \hfill summing 20 out of 100 numbers between 0 and 0.5.
                \item [e] combination of previous scenarios: \hfill summing 10 out of 500 numbers between 0 and 2.5.
                \item [f] Number of runs that did not converge.
                \item[\textdagger] \acs{mclstm} with time-dependent redistribution matrix.
            \end{tablenotes}
        \end{threeparttable}\end{center}
    \end{table*}
    
\subsubsection{Comparison with Normalized and Unitary Networks}

    In order to account for the limited range of \acp{lstm}, normalization techniques can be used to keep the data within a manageable range. 
    Therefore, we also compared \ac{mclstm} to an \ac{lstm} with layer normalization \citep{ba16layer}. 
    Although the layer normalization improves the generalization performance, it does not match the performance of \ac{mclstm} (see Table~\ref{tab:lstm_addition_full}).
    
    Whereas \ac{mclstm} preserves the $L_1$ norm, unitary \acp{rnn} \citep{arjovsky16unitary} preserve the $L_2$ norm.
    To make sure that our inductive bias on the $L_1$ norm is justified, we directly compared \ac{mclstm} to a unitary \ac{rnn}.
    We adopted the hyperparameters from \citet{arjovsky16unitary} and tried fine-tuning them, but were unable to reproduce the results on the addition problem.
    Nevertheless, we include the generalization performance of the best performing unitary \ac{rnn} in table~\ref{tab:lstm_addition_full}.
    
\subsubsection{Qualitative Analysis of the \ac{mclstm} Models Trained on Arithmetic Tasks}
\label{app:details:arithmetic:analysis}

    \paragraph{Addition Problem.}
    To reiterate, we used \ac{mclstm} with 10 hidden units and replaced the linear output layer by a simple summation.
    The model has to learn to sum all mass inputs of the timesteps, where the auxiliary input (the \emph{marker}) equals $\aux{a}\timestep{t} = 1$, and ignore all other values. 
    At the final timestep --- where the auxiliary input equals $\aux{a}\timestep{t} = -1$ --- the network should output the sum of all previously marked mass inputs.
    
    In our experiment, the model has learned to store the marked input values in a single cell, while all other mass inputs mainly end up in a single, different cell. That is, a single cell learns to accumulate the inputs to compute the solution and the other cells are used as \textit{trash cells}.
    In Fig.~\ref{app:fig:arithmetic-analysis}, we visualize the cell states for a single input sample over time, where the orange and the blue line denote the mass accumulator and the main trash cell, respectively.

    \begin{figure}[htp]
        \centering
        \includegraphics[width=.9\linewidth]{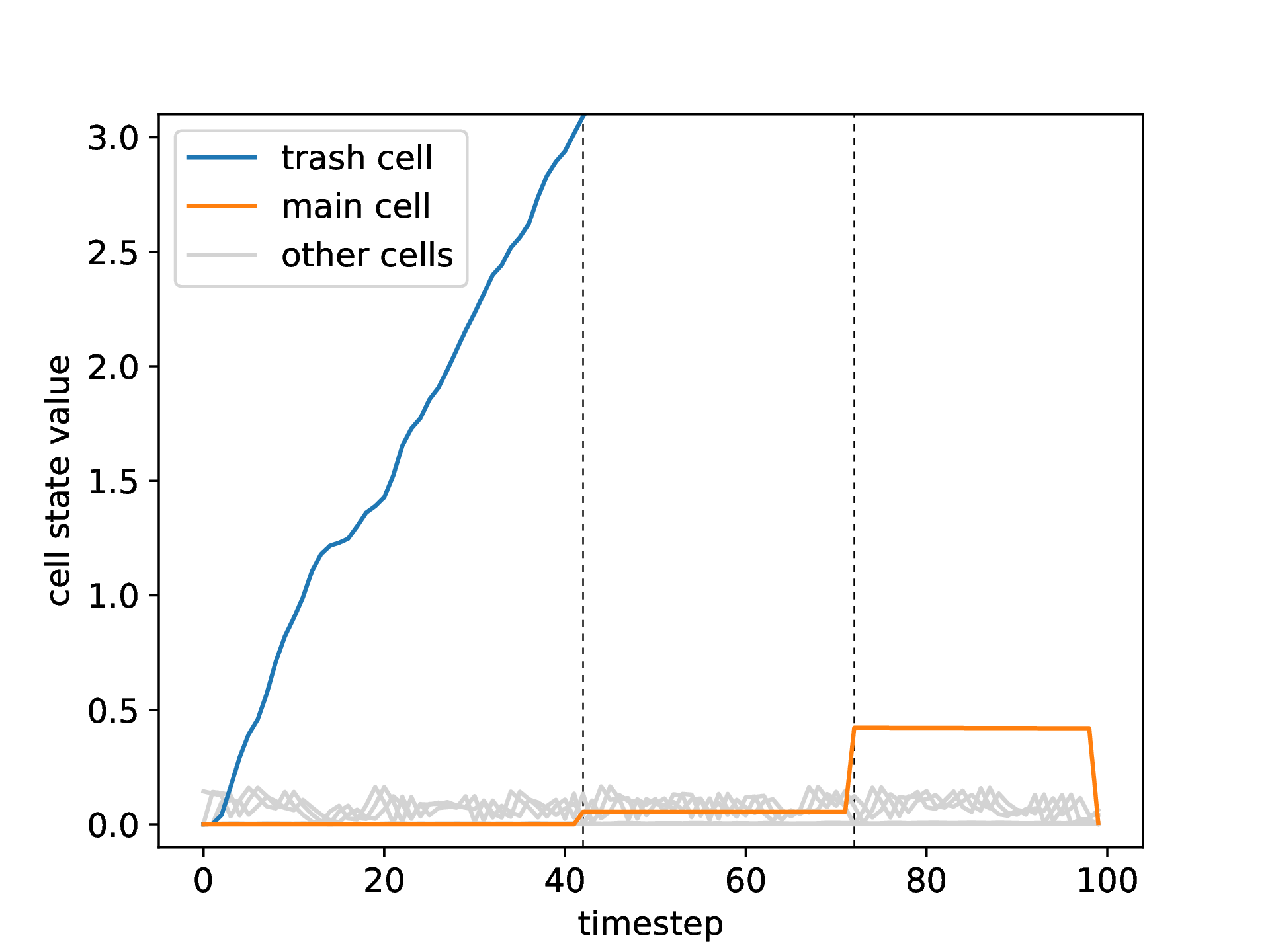}
        \caption{\ac{mclstm} cell states over time for model trained to solve the \textit{addition problem} (see Appendix \ref{app:details:arithmetic:data}). Each line denotes the value of one particular cell over time, while the two vertical grey indicator lines denote the timesteps, where the auxiliary input was 1 (i.e., which numbers in the sequence have to be added).
        \label{app:fig:arithmetic-analysis}}
    \end{figure} 
    
    We can see that at the last time step --- where the network is queried to return the accumulated sum --- the value of this mass accumulator drops to zero, i.e., the output gate is completely open. Note that this would not be the case for a model with a fully connected layer. After all, the fully connected layer can arbitrarily scale the output of the \ac{mclstm} layer, which allows the output gate to open only partially. Apart from this distinction in the last timestep, the cell states for both models behave the same way.
    For all other cells (grey lines), the output gate at the last time step is zero. 
    This illustrates nicely how the model output is only determined by the value of the single cell that acted as accumulator of the marked values (orange line).
    
    Also note the accumulating behaviour of the main trash cell (blue line).
    This can become a problem for very long sequences, because the logistic sigmoid in the output gate can never be perfectly 1 or 0.
    This means that if the value in the trash cell grows too large, it might effectively leak into the output.
    However, this undesired behaviour could be countered by, e.g., $L_2$ regularisation on the cell states, or in case of continuous prediction, adding an extra output as outlet for unnecessary mass (see Sec.~\ref{app:hyd-training-setup}).
    This should push the network to dump the trash cell to the output at timesteps where the output is not used.
    
    \paragraph{Recurrent Arithmetic.}
    In the following we take a closer look at the solution that is learned with \ac{mclstm}. Concretely, we look at the weights of a \ac{mclstm} model that successfully solves the following recurrent arithmetic  task:
    \begin{equation*}
        y = \sum_{t=1}^T (\textcolor{red}{x_6\timestep{t}} + \textcolor{cyan}{x_7\timestep{t}}) \mathbin{\square} \sum_{t=1}^T (\textcolor{cyan}{x_7\timestep{t}} + \textcolor{brown}{x_8\timestep{t}}),
    \end{equation*}
    where $\square \in \{-, +\}$, given a sequence of input vectors $\vec{x}^t \in \reals^{10}$ (the only purpose of the colors is to provide an aid to readers). We highlight the following observations: 
    
    \begin{enumerate}
    \item For the addition task (i.e., $\square \equiv +$), \ac{mclstm} has two units (see Appendix~\ref{app:details:arithmetic:hyper} for details on the experiments). \citet{traskHRRD18, madsenJ20} fixed the number of hidden units to two with the idea that each unit can learn one term of the addition operation ($\square$). However, if we take a look at the input gate of our model, we find that the first cell is used to accumulate $(x_1\timestep{t} + \ldots + x_5\timestep{t} + 0.5 \textcolor{red}{x_6\timestep{t}} + 0.5\textcolor{brown}{x_8\timestep{t}} + x_9\timestep{t} + x_{10}\timestep{t})$ and the second cell collects $(0.5 \textcolor{red}{x_6\timestep{t}} + \textcolor{cyan}{x_7\timestep{t}} + 0.5 \textcolor{brown}{x_8\timestep{t}})$. Since the learned redistribution matrix is the identity matrix, these accumulators operate individually. 
    
    This means that, instead of computing the individual terms, \ac{mclstm} directly computes the solution, scaled by a factor \nicefrac{1}{2} in its second cell. The first cell accumulates the rest of the mass, which it does not need for the prediction. In other words, it operates as some sort of \emph{trash cell}. Note that due to the mass-conservation property, it would be impossible to compute each side of the operation individually. After all, $\textcolor{cyan}{x_7\timestep{t}}$ appears on both sides of the central operation ($\square$), and therefore the data is not mass conserving.
    
   The output gate is always open for the \emph{trash cell} and closed for the other cell, indicating that redundant mass is discarded through the output of the \ac{mclstm} in every timestep and the scaled solution is properly accumulated. However, in the final timestep --- when the prediction is to be made, the output gate for the trash cell is closed and opened for the other cell. That is, the accumulated solution is passed to the final linear layer, which scales the output of \ac{mclstm} by a factor of two to get the correct solution.
    
    \item For the subtraction task (i.e., $\square \equiv -$), a similar behavior can be observed. In this case, the final model requires three units to properly generalize. The first two cells accumulate $\textcolor{red}{x_6\timestep{t}}$ and $\textcolor{brown}{x_8\timestep{t}}$, respectively. The last cell operates as \emph{trash cell} and collects $(x_1\timestep{t} + \ldots + x_5\timestep{t} + \textcolor{cyan}{x_7\timestep{t}} + x_9\timestep{t} + x_{10}\timestep{t})$. The redistribution matrix is the identity matrix for the first two cells. For the \emph{trash cell}, equal parts (0.4938) are redistributed to the two other cells. The output gate operates in a similar fashion as for addition. Finally, the linear layer computes the difference between the first two cells with weights 1, -1 and the \emph{trash cell} is ignored with weight 0.
    
    Although \ac{mclstm} with two units was not able to generalize well enough for the \citet{madsenJ20} benchmarks, it did turn out to be able to provide a reasonable solution (albeit with numerical flaws). With two cells, the network learned to store $(0.5 x_1\timestep{t} + \ldots + 0.5 x_5\timestep{t} + \textcolor{red}{x_6\timestep{t}} + 0.5 \textcolor{cyan}{x_7\timestep{t}}+ 0.5 x_9\timestep{t} + 0.5 x_{10}\timestep{t})$ in one cell, and $(0.5 x_1\timestep{t} + \ldots + 0.5 x_5\timestep{t} + 0.5 \textcolor{cyan}{x_7\timestep{t}} + \textcolor{brown}{x_8\timestep{t}} + 0.5 x_9\timestep{t} + 0.5 x_{10}\timestep{t})$ in the other cell. With a similar linear layer as for the three-unit variant, this solution should also compute a correct solution for the subtraction task.
\end{enumerate}

\subsection{Inbound-outbound Traffic Forecast}
\label{app:addons:traffic}

    Traffic forecasting considers a large number of different 
    settings and tasks \citep{tedjopurnomo2020survey}.
    For example whether the physical network topology of streets can be exploited 
    by using graph neural networks combined with \acp{lstm} \citep{cui2019traffic}. 
    Within traffic forecasting mass conservation translates to a \emph{conservation-of-vehicles} principle. 
    Generally, models that adhere to this principle 
    are desired \citep{vanajakshi2004loop,zhao2017lstm} since they
    could be useful for long-term forecasts.
    Many recent benchmarking datasets for traffic forecasts are usually uni-directional and are measured at few streets. Thus conservation laws cannot be directly applied  
    \citep{tedjopurnomo2020survey}. 

    We demonstrate how \ac{mclstm} can be used in traffic forecasting
    settings. A typical setting for vehicle conservation is 
    when traffic counts for inbound and outbound roads of a city are available. 
    In this case, all vehicles that come from an inbound road must either be within
    a city or leave the city on an outbound road. 
    The setting is similar to passenger flows in inbound and 
    outbound metro \citep{liu2019deeppf}, where \acp{lstm} have also prevailed.
    We were  able to extract such data from a recent dataset 
    based on GPS-locations \citep{kreil2020surprising} of vehicles at 
    a fine geographic grid around cities, which represents good approximation of
    a vehicle conserving scenario.
    
    
\paragraph{An approximately mass-conserving traffic dataset}

    Based on the data for the traffic4cast 2020 challenge \citep{kreil2020surprising}, we constructed a dataset to model inbound and outbound traffic of three different cities: Berlin, Istanbul and Moscow. The original data consists of 181 sequences of multi-channel images encoding traffic volume and speed for every five minutes in four (binned) directions. Every sequence corresponds to a single day in the first half of the year. In order to get the traffic flow from the multi-channel images at every timestep, we defined a frame around the city and collected the traffic-volume data for every pixel on the border of this frame. This is illustrated in Fig.~\ref{fig:traffic_inbound_outbound}. For simplicity, we ignored the fact that a single-pixel frame might have issues with fast-moving vehicles.
    
    By taking into account the direction of the vehicles, the inbound and outbound traffic can be combined for every pixel on the border of our frame. To get a more tractable dataset, we additionally combined the pixels of the four edges of the frame to end up with eight values: four values for the incoming traffic, i.e: one for each border of the frame, and four values for the outgoing traffic. The inbound traffic would be the \emph{mass} input for \ac{mclstm} and the target outputs are the outbound traffic along the different borders. The auxiliary input is the current daytime, encoded as a value between zero and one.
    
    To model the sparsity that is often available in other traffic counting problems, we chose three time-slots (6 am, 12 pm and 6 pm) for which we use fifteen minutes of the actual measurements --- i.e., three timesteps. This could for example simulate the deployment of mobile traffic counting stations. The other inputs are imputed by the average inbound traffic over the training data, which consists of 181 days. Outputs are only available when the actual measurements are used. This gives a total of 9 timesteps per day on which the loss can be computed. For training, this dataset is randomly split in 85\% training and 15\% validation samples.
    
    During inference, all 288 timesteps of the inbound and outbound measurements are used to find out which model learned the traffic dynamics from the sparse training data best. For this purpose, we used the 18 sequences of validation data from the original dataset as test set, which are distributed across the second half of the year. In order to enable a fair comparison between \ac{lstm} and \ac{mclstm}, the data for \ac{lstm} was normalized to zero mean and unit variance for training and inference (using statistics from the training data). \ac{mclstm} does not need this pre-processing step and is fed the raw data.
    
 \paragraph{Model and Hyperparameters}
    For the traffic prediction, we used \ac{lstm} followed by a fully connected layer as baseline \citep[c.f.][]{zhao2017lstm, liu2019deeppf}. For \ac{mclstm}, we chose to enforce end-to-end mass conservation by using a \ac{mcfc} output layer, which is described in detail in Appendix~\ref{app:details:arithmetic:static}. For the initialization of the models, we refer to the details of the arithmetic experiments in Appendix~\ref{app:details:arithmetic}.
    
    For each model and for each city, the best hyperparameters were found by performing a grid search on the validation data. This means that the hyperparameters were chosen to minimize the error on the nine 5-minute intervals. For all models, the number of hidden neurons was chosen from $\{10, 50, 100\}$ and for the learning rate, the options were $\{0.100, 0.050, 0.010, 0.005, 0.001\}$. All models were trained for 2\,000 epochs using the Adam optimizer \citep{kingmaB15}.
    Additionally, we considered values in $\{0, 5\}$ for the initial value of the forget gate bias in \ac{lstm}. For \ac{mclstm}, the extra hyperparameters were the initial cell state value ($\in \{0, 100\}$) --- i.e., how much cars are in each memory cell at timestep zero --- and whether or not the initial cell state should be trained via backpropagation. The results of the hyperparameter search can be found in Tab.~\ref{tab:traffic_hyperparam}.
    
    \begin{table*}
        \caption{The hyperparameters resulting from the grid search for the traffic forecast experiment.}
        \label{tab:traffic_hyperparam}
        \begin{center}\begin{tabular}{crccccc}
            \toprule
             & & hidden & lr & forget bias & initial state & learnable state \\
             \midrule
             \multirow{2}{*}{Berlin} & \acs{lstm} & 10 & 0.01 & 0 & -- & -- \\
             & \acs{mclstm} & 100 & 0.01 & -- & 0 & True \\
             \midrule
             \multirow{2}{*}{Istanbul} & \acs{lstm} & 100 & 0.005 & 5 & -- & -- \\
             & \acs{mclstm} & 50 & 0.01 & -- & 0 & False \\
             \midrule 
             \multirow{2}{*}{Moscow} & \acs{lstm} & 50 & 0.001 & 5 & -- & -- \\
             & \acs{mclstm} & 10 & 0.01 & -- & 0 & False \\
             \bottomrule
        \end{tabular}\end{center}
    \end{table*}
    
    The idea behind tuning the initial cell state, is that unlike with \ac{lstm}, the cell state in \ac{mclstm} directly reflects the number of cars that can drive out of a city during the first timesteps. If the initial cell state is too high or too low, this might negatively affect the prediction capabilities of the model. If it would be possible to estimate the number of cars in a city at the start of the sequence, this could also be used to get better estimates for the initial cell state. However, from the results of the hyperparameter search (see Tab.~\ref{tab:traffic_hyperparam}), we might have overestimated the importance of these hyperparameters.

\paragraph{Results.}
    All models were evaluated on the test data, using the checkpoint after 2\,000 epochs for fifty runs. An example of what the predictions of both models look like for an arbitrary day in an arbitrarily chosen city is displayed in Fig.~\ref{fig:traffic}. The average \ac{rmse} and \ac{mae} are summarized in Tab.~\ref{tab:traffic4cast}. The results show that \ac{mclstm} is able to generalize significantly better than \ac{lstm} for this task. The \ac{rmse} of \ac{mclstm} is significantly better than 
 \ac{lstm} ($p$-values $4\mathrm{e}{-10}$, $8\mathrm{e}{-3}$, and $4\mathrm{e}{-10}$
    for Istanbul, Berlin, and Moscow, respectively, Wilcoxon test).


\begin{figure}
    \centering
    \includegraphics[width=0.48\textwidth]{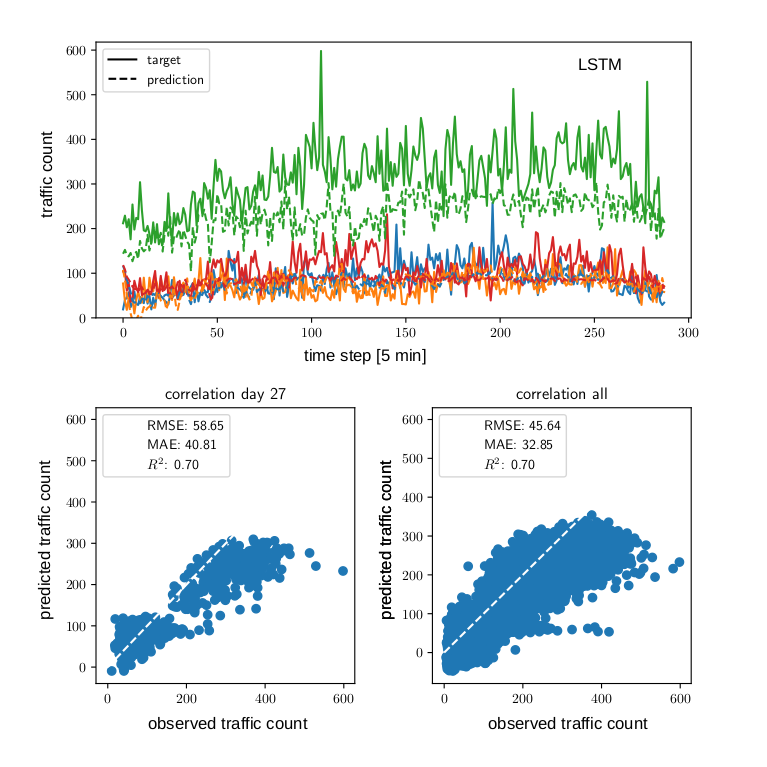}
    \includegraphics[width=0.48\textwidth]{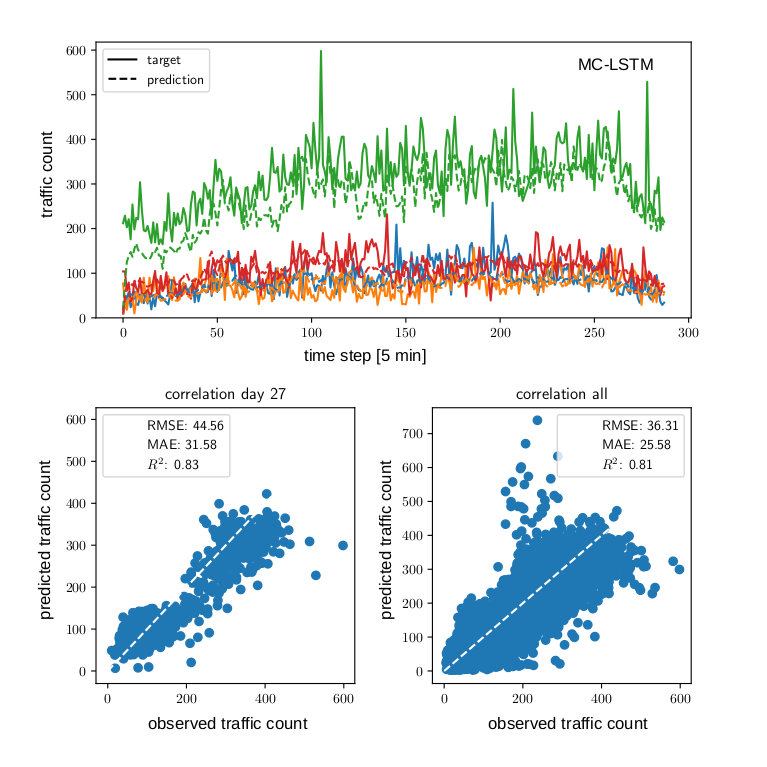}
    \caption{Traffic forecasting models for outbound traffic in Moscow. 
    An arbitrary day has been chosen for display. Note that both models have 
    only been trained on data at timesteps 71-73, 143-145, and 215-217. 
    Colors indicate the four borders of the frame, i.e., north, east, south and west. \textbf{Left:} \ac{lstm} predictions shown in dashed lines versus 
    the actual traffic counts (solid lines).
    \textbf{Right:} \ac{mclstm} predictions shown in dashed lines versus 
    the actual traffic counts (solid lines). }
    \label{fig:traffic}
\end{figure}

\begin{table*}
    \caption{Results on outbound traffic forecast avg \ac{rmse} and \ac{mae} with 95\% confidence intervals over 50 runs}
    \label{tab:traffic4cast}
    \begin{center}\begin{threeparttable}
        \begin{tabular}{rr@{$\ \pm\ $}lr@{$\ \pm\ $}lr@{$\ \pm\ $}lr@{$\ \pm\ $}lr@{$\ \pm\ $}lr@{$\ \pm\ $}l}
            \toprule
             & \multicolumn{4}{c}{Istanbul} & \multicolumn{4}{c}{Berlin} & \multicolumn{4}{c}{Moscow} \\
             \cmidrule{2-13}
             & \multicolumn{2}{c}{\acs{rmse}} & \multicolumn{2}{c}{\acs{mae}} & \multicolumn{2}{c}{\acs{rmse}} & \multicolumn{2}{c}{\acs{mae}} & \multicolumn{2}{c}{\acs{rmse}} & \multicolumn{2}{c}{\acs{mae}} \\
            \midrule
            \acs{mclstm} & \textbf{7.3} & 0.1 & \textbf{28} & 2 & \textbf{13.6} & 1.8 & \textbf{66} & 1 & 25.5 & 1.1          & \textbf{27.8} & 1.1 \\
            \acs{lstm}   & 142.6 & 4.4 &
            84 & 3 & 135.4 & 5.0 & 84 & 3 & 45.6 & 0.8 & 31.7 & 0.5 \\
            \bottomrule
        \end{tabular}
    \end{threeparttable}\end{center}
\end{table*}

\subsection{Damped Pendulum}
\label{app:addons:pendulum}

In the area of physics, we consider the problem of modeling a swinging pendulum
with friction. 
The conserved quantity of interest is the total energy.  During the movement of the pendulum, kinetic energy is converted into potential energy and vice-versa.
Neglecting friction, the total energy
is conserved and the movement would continue 
indefinitely. 
Accounting for friction, energy dissipates and 
the swinging slows over time until a fixed point is reached. This type of behavior presents a difficulty for machine learning and is impossible for 
methods that assume the pendulum to be closed systems, such as \acp{hnn} \citep{greydanusDY19}.
We postulated that both energy conversion and dissipation can be fitted by machine learning models, but that an appropriate inductive bias will allow to generalize from the learned data with more ease.   

To train the model, we generated a set of timeseries using 
the differential equations for a 
pendulum with friction. For small angles, this problem is equivalent to
the harmonic oscillator and an analytic solution exists with which 
we can compare the models \citep{iten2020discovering}. 
We used multiple different settings for initial angle, length of the pendulum, 
the amount of friction, the length of the training-period and with and without Gaussian noise. Each model received the initial 
kinetic and potential energy of the pendulum and must predict the consecutive timesteps. The time series starts always with the pendulum at the maximum displacement --- i.e., the entire energy in the system is potential energy.
We generated timeseries of potential- and kinetic energies by iterating the following settings/conditions: 
initial amplitude ($\{0.2,0.4\}$),  
pendulum length ($\{0.75, 1\}$),  
length of training sequence in terms of timesteps ($\{100, 200,400\}$),  
noise level ($\{0,0.01\}$),  
and dampening constant ($\{0.0,0.1, 0.2, 0.4,0.8\}$). 
All combinations of those settings were
used to generate a total of $120$ datasets, for which we train both models (the autoregressive \ac{lstm} and \ac{mclstm}).

We trained an autoregressive \ac{lstm} that  
receives its current state and a low-dimensional temporal 
embedding (using nine sinusoidal curves with 
different frequencies) 
to predict the potential and kinetic energy of the pendulum. 
Similarly, \ac{mclstm} is trained in an autoregressive mode, where 
a hypernetwork obtains the current state and the same temporal 
embedding as \ac{lstm}. The model-setup is thus similar to an 
autoregressive model with exogenous variables from classical timeseries modelling literature.
To obtain suitable hyperparameters we  manually adjusted the learning 
rate ($0.01$), hidden size of \ac{lstm} ($256$), the hypernetwork for estimating the redistribution (a fully connected network with 3 layers, ReLU activations and hidden sizes of 50, 100, and 2 respectively), optimizer \citep[Adam,][]{kingmaB15} and the training procedure (crucially, the amount of additionally considered timesteps in the loss after a threshold is reached. See explanation of the used loss below), on a separately generated validation dataset.

For \ac{mclstm}, a hidden size of two was used 
so that each state directly maps to the two energies. 
The hypernetwork consists of three fully connected layers of size 50, 100 
and 4, respectively. 
To account for the critical values at the 
extreme-points of the pendulum (i.e.\ the amplitudes --- where the energy is present only in the form of 
potential energy --- and the midpoint --- where only kinetic energy exists), we slightly offset the cell state from 
the actual predicted value by using a linear regression with a slope of $1.02$ and an intercept $-0.01$.

For both models, we used a combination of Pearson's correlation of 
the energy signals and the \ac{mse} as a loss function (by subtracting the former mean from the latter). Further, we used a simple curriculum to deal 
with the long autoregressive nature of the timeseries \citep{bengio2015scheduled}: 
Starting at a time window of eleven we added five additional 
timesteps whenever the combined loss was below $-0.9$.  


Overall, \ac{mclstm} has significantly outperformed \ac{lstm} with a mean \ac{mse} of $0.01$ (standard deviation $0.02$) compared to $0.07$ (standard deviation $0.14$; with a $p$-value $4.7\mathrm{e}{-10}$, Wilcoxon test).

\subsubsection{Qualitative Analysis of the \ac{mclstm} Models Trained for a Pendulum}
\label{app:Qualitative_analysis_pendulum}
In the following, we analyse the behavior of the simplest pendulum setup, i.e., the one without friction. Special to the problem of the pendulum without friction is that there are no mass in- or outputs and the whole dynamic of the system has to be modeled by the redistribution matrix. 
The initial state of the system is given by the displacement of the pendulum at the start, where all energy is stored as potential energy. 
Afterwards, the pendulum oscillates, converting potential to kinetic energy and vice-versa.

In \ac{mclstm}, the conversion between the two forms of energy has to be learned by the redistribution matrix. More specifically, the off-diagonal elements denote the fraction of energy that is converted from one form to the other. In contrast, the diagonal elements of the redistribution matrix denote the fraction of energy that is \emph{not} converted.

In Fig.~\ref{fig:pendulum_qualitative}, we visualize the off-diagonal elements of the redistribution matrix (i.e., the conversion of energy) for the pendulum task without friction, as well as the modeled potential and kinetic energy. We can see that an increasing fraction of energy is converted into the other form, until the total energy of the system is stored as either kinetic or potential energy. As soon as the total energy is e.g. converted into kinetic energy, the corresponding off-diagonal element (the orange line of the upper plot in Fig.~\ref{fig:pendulum_qualitative}) drops to zero. Here, the other off-diagonal element (the blue line of the upper plot in Fig.~\ref{fig:pendulum_qualitative}) starts to increase, meaning that energy is converted back from kinetic into potential energy. Note that the differences in the maximum values of the off-diagonal elements is not important, since at this point the corresponding energy is already approximately zero.

\begin{figure}
    \centering
    \includegraphics[width=0.9\linewidth]{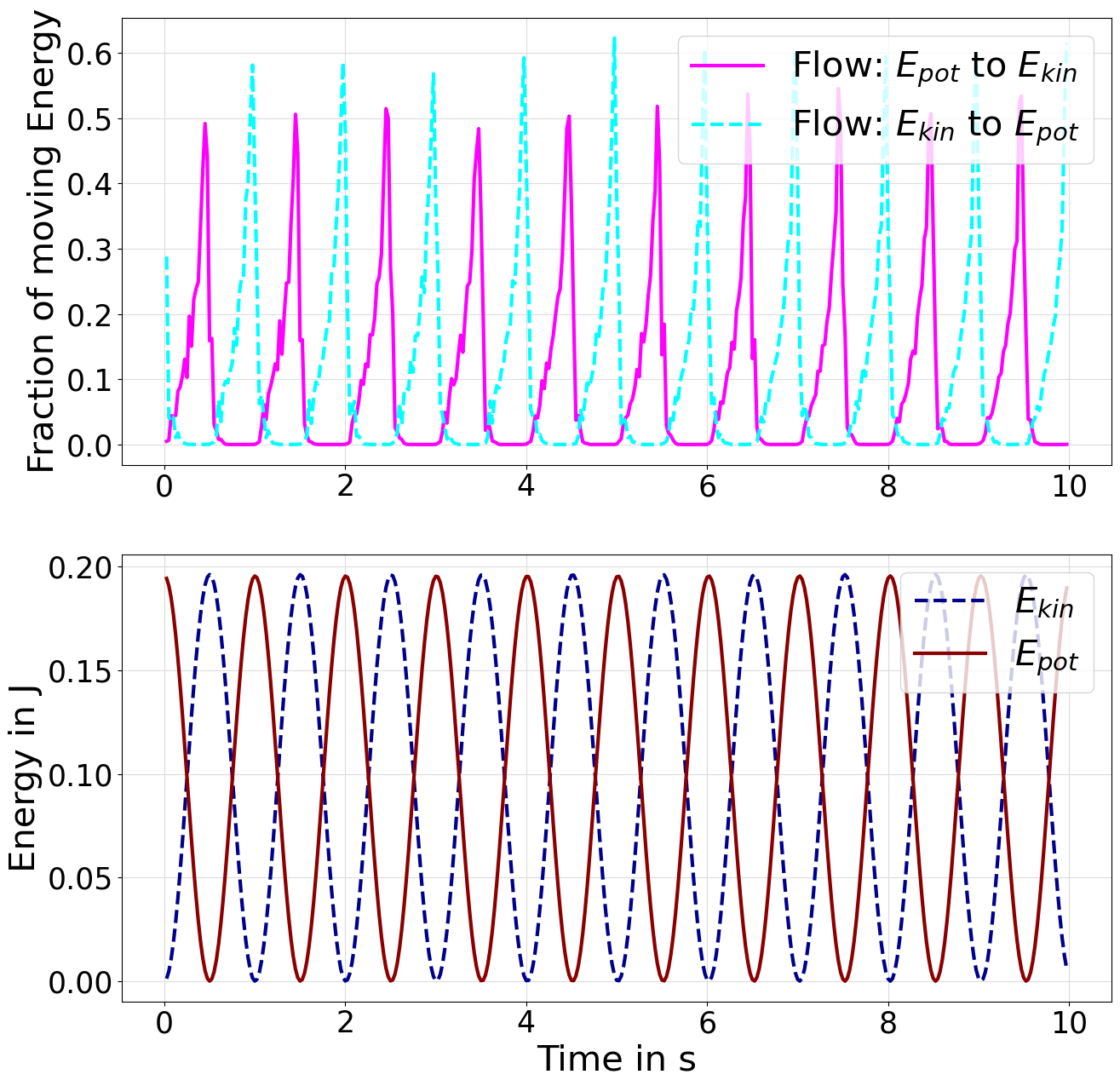}
    \caption{Redistribution of energies in a pendulum learned by MC-LSTM. The upper plot shows the fraction of energy that is redistributed between the two cells that model $E_{pot}$ and $E_{kin}$ over time. The continuous redistribution of energy results in the two time series of potential and kinetic energy displayed in the lower plot.}
    \label{fig:pendulum_qualitative}
\end{figure}


%

\subsubsection{Comparison with \aclp{hnn}}
\label{app:sec:hnn_comparison}
We aimed at a comparison with \ac{hnn} in the case
of the friction-free pendulum. To this end, we use the
data generation process by \citep{greydanusDY19}. We use amplitudes 
of $\{0.2, 0.3, 0.4, 1\}$, training sequence length $\{100, 200, 400\}$, 
and noise level $\{0,0.01\}$, which leads to 24 time-series.
We adhere to the \ac{hnn} reference implementation, which 
contains a gravity constant of $g = 6$ and mass $m = 0.5$.
In the case of the pendulum with friction, the assumptions of \acp{hnn} 
are not met which leads 
to problematic modeling behavior (see Figure~\ref{app:fig:hnn_friction}).

The \acp{hnn} directly predict the symplectic gradients that provide the 
dynamics for the pendulum. These gradients can then be integrated to obtain 
position and momentum for future timesteps. From these prediction, we 
compute the potential and kinetic energy over time.
For \ac{mclstm} we used the autoregressive version as described above
and used position and momentum, both rescaled to amplitude 1, 
as auxiliary inputs. 
Note that \acp{hnn} are feed-forward networks, and the dynamics are obtained by integrating over their predictions.
This implies that due to the periodicity of the data, 
the samples in the test set could also be in the training data.
Moreover, there is only noise on the input data, i.e., position and momentum, 
but not on the time derivatives, such that \acp{hnn} 
receive non-noisy labels. 
Therefore the training could be considered 
less noisy for \ac{hnn} compared to \acp{mclstm}.
The mean-squared error of the predictions for the  potential and kinetic energy
is compared against the analytic solution.
Concretely, the average MSE of \ac{mclstm} 
is $4.3\mathrm{e}{-4}$,  and the MSE of \acp{hnn} is $3.0\mathrm{e}{-4}$. 
On 11 out of 24 datasets, \ac{mclstm} outperformed \ac{hnn}, which indicates that there is no significant
difference between the two methods ($p$-value $0.84$, binomial test). 
 
\begin{figure}
    \centering
    \includegraphics[width=.9\linewidth]{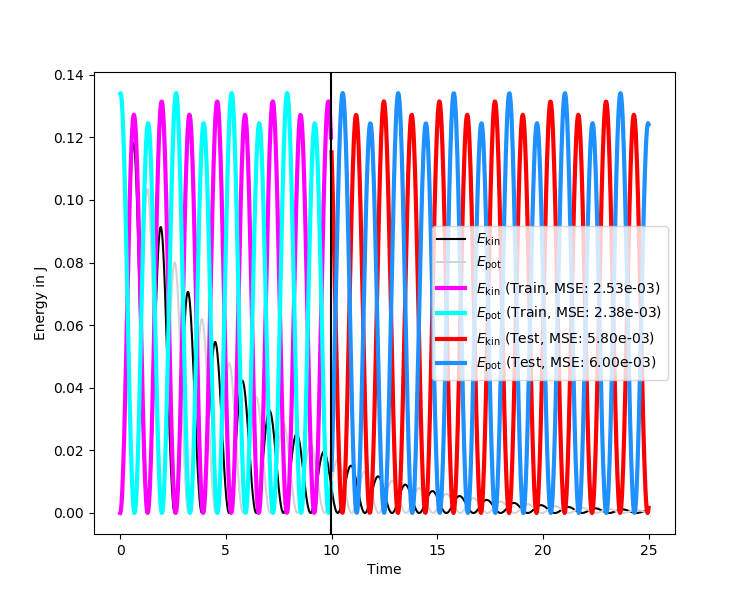}
    \caption{Example of modeling a pendulum with friction with a \ac{hnn}.
    \acp{hnn} assume a closed system and cannot model the  
    pendulum with friction, from which energy dissipates.}
    \label{app:fig:hnn_friction}
\end{figure}

\subsection{Hydrology}
\label{sec:experiments:hydrology}

Modeling river discharge from meteorological data (e.g., precipitation, temperature)
is one of the most important tasks in hydrology, and is necessary for water resource management and risk mitigation related to flooding.
Recently, \citet{kratzert2019universal, kratzert2020synergy} established \ac{lstm}-based models as 
state-of-the-art in rainfall runoff modeling, outperforming traditional hydrological models by a large 
margin against most metrics (including peak flows, which is critical for flood prediction).
However, the hydrology community is still reluctant to adopt these methods
\citep[e.g.][]{Beven2020comment}. 
A recent workshop on `Big Data and the Earth Sciences' \cite{sellars2018grand} reported that 
\textit{``[m]any participants who have worked in modeling physical-based systems continue to raise 
caution about the lack of physical understanding of ML methods that rely on data-driven approaches.”} 

One of of the most basic principles in watershed modeling is mass conservation. Whether water is treated as a resource (e.g.\ droughts) or hazard (e.g.\ floods), a modeller must be sure that they are accounting for all of the water in a catchment. 
Thus, most models conserve mass \citep{todini2008rainfallrunoff}, and attempt to explicitly implement the most important physical 
processes. 
The downside of this `model everything' strategy is that errors are introduced for every real-world process that is \textit{not} implemented in a model, or implemented incorrectly.
In contrast, \ac{mclstm} is able to learn any necessary behavior that 
can be induced from the signal (like \ac{lstm}) while still conserving the overall water budget.

\subsubsection{Details on the Dataset}

The data used in all hydrology related experiments is the publicly available Catchment Attributes and Meteorology for Large-sample Studies (CAMELS) dataset \citep{newman2014large, addor2017large}. CAMELS contains data for 671 basins and is curated by the US National Center for Atmospheric Research (NCAR). It contains only basins with relatively low anthropogenic influence (e.g., dams and reservoirs) and basin sizes range from 4 to 25\ 000 km\textsuperscript{2}. The basins cover a range of different geo- and eco-climatologies, as described by \citet{newman2015development} and \citet{addor2017camels}. Out of all 671 basins, we used 447 --- these are the basins for which simulations from all benchmark models are available (see Sec.~\ref{app:hyd:benchmark-models}). To reiterate, we used benchmark hydrology models that were trained and tested by other groups with experience using these models, and were therefore limited to the 447 basis with results for all benchmark models. The spatial distribution of the 447 basins across the contiguous USA (CONUS) is shown in Fig.~\ref{app:fig:catchments-map}.

\begin{figure*}
    \centering
    \includegraphics[width=0.8\textwidth]{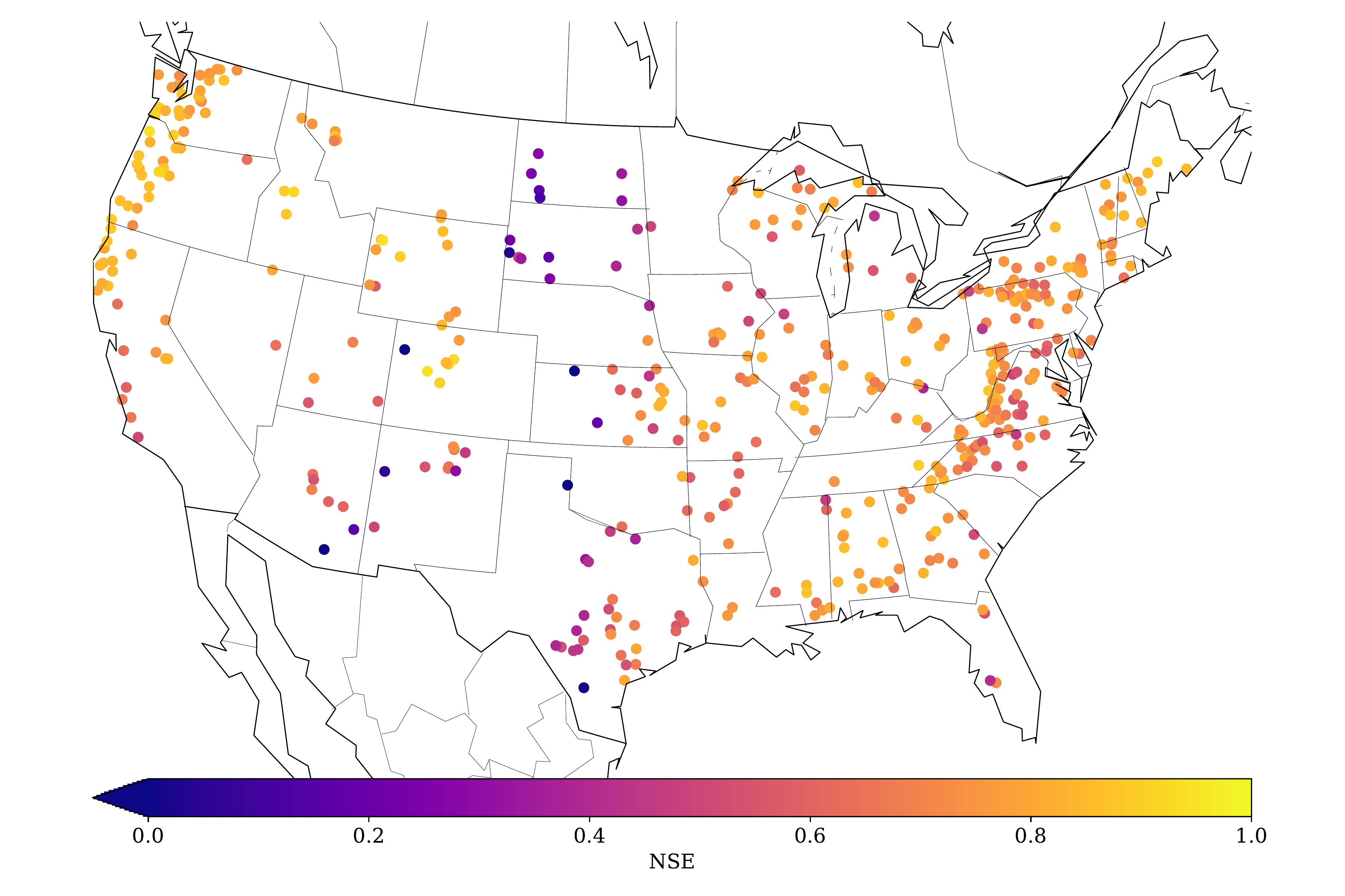}
    \caption{Spatial distribution of the 447 catchments considered in this study. The color denotes the Nash-Sutcliffe Efficiency of the MC-LSTM ensemble for each basin, where a value of 1 means perfect predictions.
    \label{app:fig:catchments-map}}
\end{figure*}

For each catchment, roughly 30 years of daily meteorological data from three different products exist (DayMet, Maurer, NLDAS). Each meteorological dataset consist of five different variables: daily cumulative precipitation, daily minimum and maximum temperature, average short-wave radiation and vapor pressure. We used the Maurer forcing data because this is the data product that was used by all benchmark models (see Sec.~\ref{app:hyd:benchmark-models}). In addition to meteorological data, CAMELS also includes a set of static catchment attributes derived from remote sensing or CONUS-wide available data products. The static catchment attributes can broadly be grouped into climatic, vegetation or hydrological indices, as well as soil and topological properties. In this study, we used the same 27 catchment attributes as \citet{kratzert2019universal}. Target data were daily averaged streamflow observations originally from the USGS streamflow gauge network, which are also included in the CAMELS dataset. 

\textbf{Training, validation and test set.} Following the calibration and test procedure of the benchmark hydrology models, we trained on streamflow observations from 1 October 1999 through 30 September 2008 and tested on observations from 1 October 1989 to 30 September 1999. The remaining period (1 October 1980 to 30 September 1989) was used as validation period for hyperparameter tuning.

\subsubsection{Details on the Training Setup and \ac{mclstm} Hyperparameters}
\label{app:hyd-training-setup}

The general model setup follows insights from previous studies \citep{kratzert2018rainfall, kratzert2019universal, kratzert2019toward, kratzert2020synergy}, where \acp{lstm} were used for the same task. We use sequences of 365 timesteps (days) of meteorological inputs to predict discharge at the last timestep of the sequence (sequence-to-one prediction). The mass input $\mass{x}$ in this experiment was catchment averaged precipitation (mm/day) and the auxiliary inputs $\aux{\vec{a}}$ were the 4 remaining meteorological variables (min. and max. temperature, short-wave radiation and vapor pressure) as well as the 27 static catchment attributes, which are constant over time.

We tested a variety of \ac{mclstm} model configurations and adaptions for this specific task, which are briefly described below: 

\begin{enumerate}
    \item \textbf{Processing auxiliary inputs with LSTM}: Instead of directly using the auxiliary inputs in the input gate (Eq.~\ref{eq:mclstm:in_gate}), output gate (Eq.~\ref{eq:mclstm:out_gate}) and time-dependent mass redistribution (Eq.~\ref{eq:R}), we first processed the auxiliary inputs \aux{a} with \ac{lstm} and then used the output of this \ac{lstm} as the auxiliary inputs. The idea was to add additional memory for the auxiliary inputs, since in its base form only mass can be stored in the cell states of \ac{mclstm}. This could be seen as a specific adaption for the rainfall runoff modeling application, since information about the weather today and in the past ought to be useful for controlling the gates and mass redistribution. Empirically however, we could not see any significant performance gain and therefore decided to not use the more complex version with an additional \ac{lstm}.
    \item \textbf{Auxiliary output + regularization to account for evapotranspiration}: Of all precipitation falling in a catchment, only a part ends as discharge in the river. Large portions of precipitation are lost to the atmosphere in form of evaporation (from e.g.\ open water surfaces) and transpiration (from e.g.\ plants and trees), and to groundwater. One approach to account for this ``mass loss'' is the following: instead of summing over outgoing mass (Eq.~\ref{eq:mclstm:out_update}), we used a linear layer to connect the outgoing mass to two output neurons. One neuron was fitted against the observed discharge data, while the second was used to estimate water loss due to unobserved sinks. A regularization term was added to the loss function to account for this. This regularization term was computed as the difference between the sum of the outgoing mass from  \ac{mclstm} and the sum over the two output neurons. This did work, and the timeseries of the second auxiliary output neuron gave interesting results (i.e.\ matching the expected behavior of the annual evapotranspiration cycle), however results were not significantly better compared to our final model setup, which is why we rejected this architectural change.
    \item \textbf{Explicit trash cell} Another way to account for evapotranspiration that we tested is to allow the model to use one memory cell as explicit ``trash cell''. That is, instead of deriving the final model prediction as the sum over the \textit{entire} outgoing mass vector, we only calculate the sum over all but e.g.\ one element (see Eq.~\ref{app:eq:hyd-trashcell}). This simple modification allows the model to use e.g.\ the first memory cell to discard mass from the system, which is then ignored for the model prediction. We found that this modification improved performance, and thus integrated it into our final model setup.
    \item \textbf{Input/output scaling to account for input/output uncertainty}: Both, input and output data in our applications inherit large uncertainties \citep{nearing2016uncertainty}, which is not ideal for mass-conserving models (and likely one of the reasons why \ac{lstm} performs so well compared to all other mass-conserving models). To account for that, we tried three different adaptions. First, we used a small fully connected network to derive time-dependent scaling weights for the mass input, which we regularized to be close to one. Second, we used a linear layer with positive weights to map the outgoing mass to the final model prediction, where all weights were initialized to one and the bias to zero. Third, we combined both. Out of the three, the input scaling resulted in the best performing model, however the results were worse than not scaling.
    \item \textbf{Time-dependent redistribution matrix variants}: For this experiment, a time-dependent redistribution matrix is necessary, since the underlying real-world processes (such as snow melt and thus conversion from snow into e.g.\ soil moisture or surface runoff) are time-dependent. Since using the redistribution matrix as proposed in Eq.~\ref{eq:R} is memory-demanding, especially for models with larger numbers of memory cells, we also tried to use a different method for this experiment. Here, we learned a fixed matrix (as in Eq.~\ref{eq:mclstm:redistribution}) and only calculated two vectors for each timestep. The final redistribution matrix was then derived as the outer product of the two time-dependent vectors and the static matrix. This resulted in lower memory consumption, however the model performance deteriorated significantly, which could be a hint toward the complexity required to learn the redistributing processes in this problem.
    \item \textbf{Activation function of the redistribution matrix}: We tested several different activation functions for the redistribution matrix in this experiment. Among those were the normalized sigmoid function, the softmax function (as in Eq.~\ref{eq:R}) and the normalized ReLU activation function (see Eq.~\ref{app:eq:hyd-relu}). We could achieve the best results using the normalized ReLU variant and can only hypothesize the reason for that: In this application (rainfall-runoff modelling) there are several state processes that are strictly disconnected. One example is snow and groundwater: groundwater will never turn into snow and snow will never transform into groundwater (not directly at least, it will first need to percolate through upper soil layers). Using normalized sigmoids or softmax makes it numerically harder (or impossible) to not distributed at least \textit{some} mass between every cell --- because activations can never be exactly zero. The normalized ReLU activation can do so, however, which might be the reason that it worked better in this case.
    \item \textbf{Activation function of the input gate}: Similar to the redistribution matrix, different activation functions can be used for the input gate. We tested the same three functions as for the redistribution matrix. For the input gate, the normalized sigmoid function resulted in the best performing model which was therefore used.
\end{enumerate}

As an extension to the standard \ac{mclstm} model introduced in Eq.~\eqref{eq:mclstm:in_gate} to 
Eq.~\eqref{eq:R},  we also used the mass input (precipitation) in all gates. The reason is the following: Different amounts of precipitations can lead to different processes. For example, low amounts of precipitation could be absorbed by the soil and stored as soil moisture, leading to effectively no immediate discharge contribution. Large amounts of precipitation on the other hand, could lead to direct surface runoff, if the water cannot infiltrate the soil at the rate of the precipitation falling down. Therefore, it is crucial that the gates have access to the information contained in the precipitation input. 
The final model design used in all hydrology experiments is described by the following equations:

\begin{align}
    \vec{m}_\mathrm{tot}\timestep{t}  &= \mat{R}\timestep{t} \cdot \vec{c}\timestep{t-1} + \vec{i}\timestep{t} \cdot \mass{x}\timestep{t} \\
    \vec{c}\timestep{t} &= (\vec{1} - \vec{o}\timestep{t}) \odot \vec{m}_\mathrm{tot}\timestep{t}  \\
    \vec{h}\timestep{t} &= \vec{o}\timestep{t} \odot \vec{m}_\mathrm{tot}\timestep{t} \\
    \label{app:eq:hyd-trashcell}
    \widehat{y} &= \sum_{i=2}^n h_i^t,
\end{align}

with the gates being defined by

\begin{align}
    \vec{i}\timestep{t} &= \normalised{\sigmoid} (\mat{W}_\mathrm{i} \cdot \aux{\vec{a}}\timestep{t} + \mat{U}_\mathrm{i} \cdot \frac{\vec{c}\timestep{t-1}}{\norm{\vec{c}\timestep{t-1}}_1} + \mat{V}_\mathrm{i} \cdot \mass{x}\timestep{t} + \vec{b}_\mathrm{i}) \\
    \vec{o}\timestep{t} &= \sigmoid(\mat{W}_\mathrm{o} \cdot \aux{\vec{a}}\timestep{t} + \mat{U}_\mathrm{o} \cdot \frac{\vec{c}\timestep{t-1}}{\norm{\vec{c}\timestep{t-1}}_1} + \mat{V}_\mathrm{o} \cdot \mass{x}\timestep{t} + \vec{b}_\mathrm{o}) \\
    \mat{R}\timestep{t} &= \widetilde{\relu} \left( \ten{W}_\mathrm{r} \cdot \aux{\vec{a}}\timestep{t} + \ten{U}_\mathrm{r} \cdot \frac{\vec{c}\timestep{t-1}}{\norm{\vec{c}\timestep{t-1}}_1}+ \ten{V}_\mathrm{r} \cdot \mass{x}\timestep{t} + \mat{B}_\mathrm{r} \right),
\end{align}

where $\normalised{\sigmoid}$ is the \emph{normalized logistic function}
and $\widetilde{\relu}$ is the \emph{normalized rectified linear unit} (ReLU) 
that we define in the following. 
The normalized logistic function defined of the input gate 
is defined by:

\begin{equation}
    \normalised{\sigmoid}(i_{k}) = \frac{\sigmoid(i_{k})}{\sum_k \sigmoid(i_{k})}.
\end{equation}

In this experiment, the activation function for the redistribution 
gate is the normalized ReLU function defined by:

\begin{equation}\label{app:eq:hyd-relu}
    \widetilde{\relu}(s_{k}) = \frac{\max(s_{k}, 0)}{\sum_k \max(s_{k}, 0)}, 
\end{equation}

where $\vec{s}$ is some input vector to the normalized ReLU function.

We manually tried different sets of hyperparameters, because a large-scale automatic hyperparameter search was not feasible. Besides trying out all variants as described above, the main hyperparameter that we tuned for the final model was the number of memory cells. For other parameters, such as learning rate, mini-batch size, number of training epochs, we relied on previous work using \acp{lstm} on the same dataset.

The final hyperparameters are a hidden size of 64 memory cells and a mini-batch size of 256. We used the Adam optimizer \citep{kingmaB15} with a scheduled learning rate starting at 0.01 then lowering the learning rate after 20 epochs to 0.005 and after another 5 epochs to 0.001. We trained the model for a total number of 30 epochs and used the weights of the last epoch for the final model evaluation. All weight matrices were initialized as (semi) orthogonal matrices \citep{saxe2013exact} and all bias terms with a constant value of zero. The only exception was the bias of the output gate, which we initialized to $-3$, to keep the output gate closed at the beginning of the training.

\subsubsection{Details on the Evaluation Metrics}

Table \ref{app-hydrology-metric-definition} lists the definition of all metrics used in the hydrology experiments as well as the corresponding references.

\begin{table*}
\caption{Definition of all metrics used in the hydrology experiments. The NSE is defined as the $R^2$ between simulated, $\hat{y}$, and observed, $y$, runoff and is listed for completion. FHV and FLV are both derived from the flow duration curve, which is a cumulative frequency curve of the discharge. $H$ for the FHV and $L$ for the FLV correspond to the 2\% highest flow and the 30\% lowest flow, respectively. \label{app-hydrology-metric-definition}}
\begin{center}\begin{threeparttable}
\begin{tabular}{lcc}
\toprule
Metric & Reference & Equation  \\
\midrule
Nash-Sutcliff-Efficiency (NSE)\textsuperscript{a} & \citet{nash1970river} & $1 - \frac{\sum_{t=1}^T (\widehat{y}\timestep{t} - y\timestep{t})^2}{\sum_{t=1}^T (y\timestep{t} - \bar{y})^2}$ \\
$\beta$-NSE Decomposition\textsuperscript{b} & \citet{gupta2009decomposition} & $(\mu_{\widehat{y}} - \mu_y) / \sigma_y$ \\
Top 2\% peak flow bias (FHV)\textsuperscript{c} & \citet{yilmaz2008} & $\frac{\sum_{h=1}^H (\widehat{y}_h - y_h)}{\sum_{h=1}^H y_h} \times 100$\\
30\% low flow bias (FLV)\textsuperscript{d} & \citet{yilmaz2008} & $\frac{\sum_{l=1}^L(\textup{log}(\widehat{y}_{l})-\textup{log}(\widehat{y}_{L})) - \sum_{l=1}^L(\textup{log}(y_{l})-\textup{log}(y_{L}))}{\sum_{l=1}^L(\textup{log}(y_{l})-\textup{log}(y_{L}))} \times 100$ \\
\bottomrule
\end{tabular}
\begin{tablenotes}[para]
\small{
\textsuperscript{a}: \textit{Nash-Sutcliffe efficiency: $(-\infty, 1]$, values closer to one are desirable.} \\
\textsuperscript{b}: \textit{$\beta$-NSE decomposition: $(-\infty, \infty)$, values closer to zero are desirable.}\\
\textsuperscript{c}: \textit{Top 2\% peak flow bias: $(-\infty, \infty)$, values closer to zero are desirable.} \\
\textsuperscript{d}: \textit{Bottom 30\% low flow bias: $(-\infty, \infty)$, values closer to zero are desirable.}\\
}
\end{tablenotes}
\end{threeparttable}\end{center}
\end{table*}

\subsubsection{Details on the LSTM Model}
For \ac{lstm}, we largely relied on expertise from previous studies \citep{kratzert2018rainfall, kratzert2019universal, kratzert2019toward, kratzert2020synergy}. The only hyperparameter we adapted was the number of memory cells, since we used fewer basins (447) than in the previous studies (531). We found that \ac{lstm} with 128 memory cells, compared to the 256 used in previous studies, resulted in slightly better results. Apart from that, we trained \acp{lstm} with the same inputs and settings (sequence-to-one with a sequence length of 365) as described in the previous section for \ac{mclstm}.
We used the standard \ac{lstm} implementation from the PyTorch package \citep{Paszke2019pytorch}, i.e., with forget gate \citep{gersSC00}. We manually initialized the bias of the forget gate to be 3 in order to keep the forget gate open at the beginning of the training.

\subsubsection{Details on the Benchmark Models}
\label{app:hyd:benchmark-models}

The benchmark models were first collected by \citet{kratzert2019universal}. All models were configured, trained and run by several different research groups, most often the respective model developers themselves. This was done to avoid any potential to favor our own models. All models used the same forcing data (Maurer) and the same time periods to train and test. The models can be classified in two groups:

\begin{enumerate}
    \item \textit{Models trained for individual watersheds}. These are SAC-SMA \citep{newman2017benchmarking}, VIC \citep{newman2017benchmarking}, three different model structures of FUSE\footnote{Provided by Nans Addor on personal communication}, mHM \citep{mizukami2019choice} and HBV \citep{seibert2018upper}. For the HBV model, two different simulations exist: First, the ensemble average of 1000 untrained HBV models (lower benchmark) and second, the ensemble average of 100 trained HBV models (upper benchmarks). For details see \citep{seibert2018upper}.
    \item \textit{Models trained regionally}. For hydrological models, regional training means that one parameter transfer model was trained, which estimates watershed-specific model parameters through globally trained model functions from e.g.\ soil maps or other catchment attributes. For this setting, the benchmark dataset includes simulations of the VIC model \citep{mizukami2017towards} and mHM \citep{rakovec2019diagnostic}.
\end{enumerate}

\subsubsection{Detailed Results}
Table~\ref{app:tab:hyd-mclstm-lstm-comparison} provides results for \ac{mclstm} and \ac{lstm} averaged over the $n=10$ model repetitions.

\begin{table*}
\caption{Model robustness of \ac{mclstm} and \ac{lstm} results over the $n=10$ different random seeds. For all $n=10$ models, we calculated the median performance for each metric and report the mean and standard deviation of the median values in this table.\label{app:tab:hyd-mclstm-lstm-comparison}}
\begin{center}\begin{threeparttable}
\begin{tabular}{lccccc}
\toprule
{} &         MC\textsuperscript{a} &       NSE\textsuperscript{b} &  $\beta$-NSE\textsuperscript{c} &        FLV\textsuperscript{d} &        FHV\textsuperscript{e} \\
\midrule
MC-LSTM Single   &  \ding{51} &  0.726$\pm$0.003 & -0.021$\pm$0.003 & -38.7$\pm$3.2 & -13.9$\pm$0.7 \\
LSTM Single      &  \ding{55} &  0.737$\pm$0.003 & -0.035$\pm$0.005 &  13.6$\pm$3.4 & -14.8$\pm$1.0 \\
\bottomrule
\end{tabular}
\begin{tablenotes}[para]
\small{
\textsuperscript{a}: \textit{Mass conservation (MC)}. \\
\textsuperscript{b}: \textit{Nash-Sutcliffe efficiency: $(-\infty, 1]$, values closer to one are desirable.} \\
\textsuperscript{c}: \textit{$\beta$-NSE decomposition: $(-\infty, \infty)$, values closer to zero are desirable.}\\
\textsuperscript{d}: \textit{Bottom 30\% low flow bias: $(-\infty, \infty)$, values closer to zero are desirable.}\\
\textsuperscript{e}: \textit{Top 2\% peak flow bias: $(-\infty, \infty)$, values closer to zero are desirable.}\\
}
\end{tablenotes}
\end{threeparttable}\end{center}
\end{table*}

\subsection{Ablation Study}
In order to demonstrate that the design choices of \ac{mclstm} are necessary together
to enable accurate predictive models, we performed an ablation study. In this study, we make the following changes to the input gate, the redistribution operation, and the output gate, to test if mass conservation in the individual parts is necessary.

\begin{enumerate}
    \item Input gate: We change the activation function of the input gate from a normalized sigmoid function to the standard sigmoid function, thus resulting in the input gate of a standard \ac{lstm}. Since the sigmoid function is bounded to $(0,1)$, the mass input $\mass{x}$ at every timestep $t$ that is added into the system can be scaled between $(0, n*\mass{x}\timestep{t})$.
    \item Redistribution matrix: We remove the normalized activation function from the redistribution matrix and instead use a linear activation function. This allows for unconstrained and scaled flow of mass from each memory cell into each other memory cell.
    \item Output gate: Instead of removing the outgoing mass ($\vec{o}\timestep{t} \odot \vec{m}_\mathrm{tot}\timestep{t}$) from the cell states at each timestep $t$, we leave the cell states unchanged and keep all mass within the system.
\end{enumerate}

We test these variants on data from the hydrology experiment. We chose 5 random basins to limit computational expenses and trained nine repetitions for each configuration and basin. The results are compared against the full mass-conserving
\ac{mclstm} architecture as described in App. \ref{app:hyd-training-setup}
and reported in Table~\ref{tab:ablation}. The results of
the ablation study indicate that the design of the input gate, redistribution matrix, 
and output gate, are necessary together for proficient predictive performance. The
strongest decrease in performance is observed if redistribution matrix does not conserve
mass, and smaller decreases if input or output gate do not conserve mass. 
We also tested a variant, where we used the softmax activation function in the input gate, instead of the normalized sigmoid that was used in the hydrology experiments (see Sec.~\ref{app:hyd-training-setup}). Both mass conserving variants, once with normalized sigmoid as activation function and once with softmax, achieve similar performance, while the variant with softmax is slightly better.
However, as stated in Sec.~\ref{app:hyd-training-setup}, we also tested this variant on the multi-basin version that we trained for the hydrology experiments. 
Here, the normalized sigmoid activation function resulted in better model performance.
This emphasizes that both variants are viable options and the exact design of the \acs{mclstm} might be task dependent.

\begin{table}
\caption{Ablation study results of the hydrology experiment. Models are trained for five, random basin with nine model repetitions. We computed the median over the repetitions and then the mean over the five basins. \label{tab:ablation}}
\begin{center}\begin{threeparttable}
\begin{tabular}{lcc}
\toprule
{} &         MC\textsuperscript{a} &       NSE\textsuperscript{b} \\
\midrule
\acs{mclstm}   &  \ding{51} &  $0.635 \pm 0.102$ \\
\acs{mclstm} w. softmax & \ding{51} & $0.650 \pm 0.095$ \\
\acs{mclstm} $-$ input     &  \ding{55} & $0.603 \pm 0.123$ \\
\acs{mclstm} $-$ output     &  \ding{55} & $0.55 \pm 0.097$ \\
\acs{mclstm} $-$ redis.\textsuperscript{c}     &  \ding{55} & $-4.229 \pm 8.982$ \\
\bottomrule
\end{tabular}
\begin{tablenotes}[para]
\small{
\textsuperscript{a}: \textit{Mass conservation (MC)}. \\
\textsuperscript{b}: \textit{Nash-Sutcliffe efficiency: $(-\infty, 1]$, values closer to one are desirable.}\\
\textsuperscript{c}: \textit{For one out of five basins, all nine model repetitions resulted into NaNs during training. Here, we report the statistics calculated from only the four successful basins.}  \\
}
\end{tablenotes}
\end{threeparttable}\end{center}
\end{table}

\subsection{Runtime}
\label{app:runtime}

Section~\ref{sec:theory} provides a comparison of \ac{mclstm} and \ac{lstm} in terms of computational complexity.
Since this comparison is rather abstract, we also conducted an empirical evaluation of the runtime.
The empirical runtimes of the forward pass for a single batch for both \ac{mclstm} and \ac{lstm} are listed in table~\ref{tab:runtime}.
Note that the backward pass should scale similarly to the forward pass.

\begin{table}
    \centering
    \begin{tabular}{r|ll}
        & CPU & GPU \\
        \hline \rule{0pt}{2.6ex}
        \ac{mclstm} & $951^{+1}_{-5}$ & $236^{+16}_{-1}$ \\
        \ac{lstm} & $205^{+8}_{-10}$ & $121^{+0}_{-0}$
    \end{tabular}
    \caption{Median runtime in ms of 5 forward passes with indication of 25 and 75\% quantiles. Timings were executed on a PC with AMD Ryzen 7 2700 CPU and Nvidia GTX 1070Ti GPU.}
    \label{tab:runtime}
\end{table}

We used the prototypical architecture for the hydrology experiments.
Concretely, both models received 1 mass input, 30 auxiliary inputs and had 64 hidden units.
A batch of 256 sequences was used, where each sequence has 365 timesteps.
To keep the comparison fair, we used a custom LSTM rather than the highly optimised default implementation that is available in pytorch \citep{Paszke2019pytorch}.


\section{Theorems \& Proofs}
\label{app:proof}
\setcounter{theorem}{0}

\begin{theorem}[Conservation property]
    Let $m_c\timestep{\tau} = \sum_k c_k\timestep{\tau}$ and $m_h\timestep{\tau} = \sum_k h_k\timestep{\tau}$ be, respectively, the \emph{mass} in the \ac{mclstm} storage and the outputs at time $\tau$. At any timestep $\tau$, we have:
    \begin{equation*}
        m_c\timestep{\tau} = m_c\timestep{0} + \sum_{t=1}^\tau \mass{x}\timestep{t}- \sum_{t = 1}^\tau m_h\timestep{t}.
    \end{equation*}
    That is, the change of mass in the cell states is the difference between input and output mass, accumulated over time.
\end{theorem}

\begin{proof}
    The proof is by induction and we use  $\vec{m}_{\mathrm{tot}} = \mat{R}\timestep{t} \cdot \vec{c}\timestep{t-1} + \vec{i}\timestep{t} \cdot \mass{x}\timestep{t}$ from 
    Eq.\eqref{eq:mclstm:mass}.
    
    For $\tau = 0$, we have $m_c\timestep{0} = m_c\timestep{0} + \sum_{t = 1}^0 \mass{x}\timestep{t}- \sum_{t = 1}^0 m_h\timestep{t}$, which is trivially true when using the convention that $\sum_{t = 1}^0 = 0$.
    
    Assuming that the statement holds for $\tau = T$, we show that it must also hold for $\tau = T + 1$.
    
    Starting from Eq.~\eqref{eq:mclstm:cell_update}, the mass of the cell states at time $T + 1$ is given by:
    \begin{equation*}
        m_c\timestep{T+1}= \sum_{k=1}^K (1 - o_k) m_{\mathrm{tot},k}\timestep{T+1} 
        = \sum_{k=1}^K m_{\mathrm{tot},k}\timestep{T+1} - \sum_{k=1}^K o_k m_{\mathrm{tot},k}\timestep{T+1},
    \end{equation*}
    where $m_{\mathrm{tot}, k}\timestep{t}$ is the $k$-th entry of the result from Eq.~\eqref{eq:mclstm:mass} (at timestep $t$). The sum over entries in the first term can be simplified as follows:
    \begin{align*}
        \sum_{k=1}^K m_{\mathrm{tot},k}\timestep{T+1} &= \sum_{k=1}^K \left(\sum_{j=1}^K r_{kj} c_j\timestep{T} + i_k \mass{x} \timestep{T+1}\right) \\
        &= \sum_{j=1}^K c_j\timestep{T} \left(\sum_{k=1}^K r_{kj}\right) + \mass{x}\timestep{T+1}\sum_{k=1}^K i_k \\
        &= m_c\timestep{T} + \mass{x}\timestep{T+1}.
    \end{align*}
    The final simplification is possible because $\mat{R}$ and $\vec{i}$ are (left-)stochastic.
    The mass of the outputs can then be computed from Eq.~\eqref{eq:mclstm:out_update}:
    \begin{equation*}
        m_h\timestep{T+1}= \sum_{k=1}^K o_k m_{\mathrm{tot},k}\timestep{T+1}.
    \end{equation*}
    
    Putting everything together, we find
    \begin{align*}
        m_c\timestep{T+1}&= \sum_{k=1}^K m_{\mathrm{tot},k}\timestep{T+1} - \sum_{k=1}^K o_k m_{\mathrm{tot},k}\timestep{T+1} \\
        &= m_c\timestep{T} + \mass{x}\timestep{T+1}- m_h\timestep{T+1}\\
        &= m_c\timestep{0} + \sum_{t=1}^T \mass{x}\timestep{t}- \sum_{t=1}^T m_h\timestep{t}+ \mass{x}\timestep{T+1}- m_h\timestep{T+1}\\
        &= m_c\timestep{0} + \sum_{t=1}\timestep{T+1} \mass{x}\timestep{t}- \sum_{t=1}^{T+1} m_h\timestep{t}
    \end{align*}
    
    By the principle of induction, we conclude that mass is conserved, as specified in Eq.~\eqref{eq:massconv}.
\end{proof}

\begin{corollary}
    \label{cl:bound}
    In each timestep $\tau$, the cell states $c_k\timestep{\tau}$
    are bounded by the sum of mass inputs $\sum_{t=1}^\tau \mass{x}\timestep{\tau} + m_c\timestep{0}$, that 
    is $|c_k\timestep{\tau}| \leq \sum_{t=1}^\tau \mass{x}\timestep{\tau} + m_c\timestep{0}$. Furthermore, 
    if the series of mass inputs converges $\lim_{\tau \rightarrow \infty} \sum_{t=1}^\tau \mass{x}\timestep{\tau} = m_x^\infty$, then also the sum of cell states 
    converges.
\end{corollary}

\begin{proof}
    Since $c_k \timestep{t} \geq 0$, $\mass{x}\timestep{t} \geq 0$ and $m_h\timestep{t} \geq 0$ 
    for all $k$ and $t$,
    \begin{align}
     |c_k\timestep{\tau}| = c_k\timestep{\tau} \leq \sum_{k=1}^K c_k\timestep{\tau} = 
     m_c\timestep{\tau} \leq \sum_{t=1}^\tau \mass{x}\timestep{\tau} + m_c\timestep{0}, 
    \end{align}
    
    where we used Theorem~1. Convergence follows immediately through 
    the \emph{comparison test}. 
\end{proof}

\section{On Random Markov Matrices.}
\label{sec:circular}
When initializing an \ac{mclstm} model, the entries of 
the redistribution matrix $\mat{R}$ of dimension $K\times K$ 
are created from non-negative and iid random variables $(s_{ij})_{1 \le i,j \le K}$ with finite means $m$ and variances $\sigma^2$ and bounded 
fourth moments. We collect them in a matrix $\mat{S}$. Next we assume that those entries 
get column-normalized to obtain the random Markov matrix $\mat{R}$. 

\paragraph{Properties of Markov matrices and random Markov matrices.}
Let $\lambda_1,\ldots,\lambda_K$ be the eigenvalues and 
$s_1,\ldots, s_K$ be the singular values of $\mat{R}$, ordered such that
$|\lambda_1| \geq \ldots \geq |\lambda_K|$ and $s_1 \geq \ldots \geq s_k$. 
We then have the following properties for any Markov matrix (not necessarily random):

\begin{itemize}
    \item $\lambda_1=1$. 
    \item $\vec{1}^T \mat{R}=\vec{1}^T$.
    \item $s_1=\norm{\mat{R}}_2 \leq \sqrt{K}$.
    \end{itemize}
    Furthermore, for random Markov matrices, we have
\begin{itemize}
    \item $\lim_{K \rightarrow \infty} s_1 = 1$ \citep[Theorem 1.2]{bordenave10circular}
\end{itemize}
For the reader's convenience we briefly discuss further selected interesting properties of random Markov matrices in the next paragraph, especially concerning the global behavior of their eigenvalues and singular values.

\paragraph{Circular and Quartercircular law for random Markov matrices.}
In random matrix theory one major field of interest concerns the behavior of eigenvalues and singular values when $K \to \infty$. One would like to find out how the limiting distribution of the eigenvalues or singular values looks like. To discuss the most important results in this direction for large Markov matrices $\mat{R}$, let us introduce some notation. 
\begin{itemize}
\item $\delta_a$ denotes the Dirac delta measure centered at $a$. \item By
$
    \mu_{\mat{R}}=\frac{1}{K} \sum_{k=1}^K \delta_{\lambda_k}
$
we denote the empirical spectral density of the eigenvalues of $\mat{R}$. 
\item Similarly we define the empirical spectral density of the singular values of $\mat{R}$ as:
$
    \nu_{\mat{R}}=\frac{1}{K} \sum_{k=1}^K \delta_{s_k}.
$
\item $\mathcal{Q}_{\sigma}$ denotes the quartercircular distribution on the interval $[0,\sigma]$ and 
\item $\mathcal{U}_{\sigma}$ the uniform distribution on the disk $\{ z \in \mathbb{C}: |z| \le \sigma \}$.
\end{itemize}
Then we have as $K \to \infty$:
\begin{itemize}
    \item \textbf{Quarter cirular law theorem:} \citep[Theorem 1.1]{bordenave10circular}: $ \nu_{\sqrt{K} \mat{R}} \to \mathcal{Q}_{\sigma}$ almost surely.
    \item \textbf{Cirular law theorem:} \citep[Theorem 1.3]{bordenave10circular}: $\nu_{\sqrt{K} \mat{R}} \to \mathcal{U}_{\sigma}$ almost surely.
\end{itemize}
The convergence here is understood in the sense of weak convergence of probability measures with respect
to bounded continuous functions. Note that those two famous theorems originally appeared for $\frac{1}{\sqrt{K}}\mat{S}$ instead of $\sqrt{K}\mat{R}$. Of course much more details on those results can be found in \citet{bordenave10circular}.

\paragraph{Gradient flow of \ac{mclstm} for random redistributions.}
Here we provide a short note on the gradient dynamics of the cell state
in a random \ac{mclstm}, hence, at initialization of the model. 
Specifically we want to provide some heuristics based on the arguments about the 
behavior of large stochastic matrices.
Let us start by recalling the formula for $\vec{c}\timestep{t}$:
\begin{align}
    \vec{c}\timestep{t} &= (\vec{1} - \vec{o}\timestep{t}) \odot (\mat{R}\timestep{t} \cdot \vec{c}\timestep{t-1} + \vec{i}\timestep{t} \cdot \mass{x}\timestep{t}).
\end{align}

Now we investigate the gradient of $\norm{\frac{\partial \vec{c}\timestep{t}}{\partial \vec{c}\timestep{t-1}}}_2$ in the limit $K \to \infty$. We assume that for $K \to \infty$,  $\vec{o}\timestep{t} \approx \vec{0}$ and $\vec{i}\timestep{t} \approx \vec{0}$ for all $t$. 
Thus we approximately have:
\begin{align}
    \norm{\frac{\partial \vec{c}\timestep{t}}{\partial \vec{c}\timestep{t-1}}}_2 \approx \norm{\mat{R}\timestep{t}}_2.
\end{align}
$\mat{R}\timestep{t}$ is a stochastic matrix, and $s_1 = \norm{\mat{R}\timestep{t}}_2$ is its largest singular value. Theorem 1.2 from \citet{bordenave10circular} ensures that $\norm{\mat{R}\timestep{t}}_2=1$ for $K \to \infty$ under reasonable moment assumptions on the distribution of the unnormalized entries (see above). 
Thus we are able to conclude $\norm{\frac{\partial \vec{c}\timestep{t}}{\partial \vec{c}\timestep{t-1}}}_2 \approx 1$  for large $K$ and all $t$, which can 
prevent the gradients from exploding.


\end{document}